\documentclass[accepted]{uai2025} 
                        
\usepackage[american]{babel}

\usepackage{url}

\usepackage[T1]{fontenc}    
\usepackage{booktabs}       
\usepackage{amsfonts}       
\usepackage{nicefrac}       
\usepackage{microtype}      
\usepackage{xcolor}         
\usepackage{graphicx}

\usepackage{amsfonts}
\usepackage{amsmath}
\usepackage{amssymb}
\usepackage{amsthm}

\usepackage{dsfont}
\usepackage{accents}
\usepackage{nicefrac}
\usepackage{xcolor}

\usepackage{algorithm}
\usepackage{algorithmic}

\usepackage{cleveref}
\usepackage{multirow}
\usepackage{caption}
\usepackage{subcaption}
\usepackage{enumitem}

\usepackage{makecell}

\usepackage{natbib} 
\renewcommand{\bibsection}{\subsubsection*{References}}

\ifodd 1
\newcommand{\jingc}[1]{{\color{magenta}(Jing: #1)}}
\else
\newcommand{\jingc}[1]{}
\fi

\ifodd 1
\newcommand{\jing}[1]{{\color{magenta} #1}}
\else
\newcommand{\jing}[1]{}
\fi

\usepackage{bbm}

\newcommand{\cA}{{\mathcal{A}}}

\newcommand{\Dc}{{\mathcal{D}}}

\newcommand{\Mc}{{\mathcal{M}}}
\newcommand{\Xc}{{\mathcal{X}}}

\newcommand{\Eb}{{\mathbb{E}}}

\newcommand{\Rb}{{\mathbb{R}}}

\newcommand{\TT}{\intercal}

\theoremstyle{plain}
\newtheorem{theorem}{Theorem}[section]

\newtheorem{lemma}{Lemma}[section]
\newtheorem{corollary}{Corollary}[section]
\theoremstyle{definition}
\newtheorem{definition}{Definition}[section]

\theoremstyle{remark}
\newtheorem{remark}{Remark}[section]
\theoremstyle{example}

\title{Augmenting Online RL with Offline Data is All You Need:\\ A Unified Hybrid RL Algorithm Design and Analysis}

\author[1]{Ruiquan Huang\thanks{Equal contribution.}}
\author[2]{Donghao Li$^*$}
\author[2]{Chengshuai Shi}
\author[2]{Cong Shen}
\author[2]{Jing Yang}
\affil[1]{%
    Electrical Engineering Dept.\\
    The Pennsylvania State University\\
    State College, Pennsylvania, USA
}
\affil[2]{%
    Electrical and Computer Engineering Dept.\\
    University of Virginia\\
    Charlottesville, Virginia, USA\\
}

\begin{document}
\maketitle

\begin{abstract}
  This paper investigates a hybrid learning framework for reinforcement learning (RL) in which the agent can leverage both an offline dataset and online interactions to learn the optimal policy. We present a unified algorithm and analysis and show that augmenting confidence-based online RL algorithms with the offline dataset outperforms any pure online or offline algorithm alone and achieves state-of-the-art results under two learning metrics, i.e., sub-optimality gap and online learning regret. Specifically, we show that our algorithm achieves a sub-optimality gap $\tilde{O}( \sqrt{1/(N_0/\mathtt{C}(\pi^*|\rho)+N_1}) )$, where $\mathtt{C}(\pi^*|\rho)$ is a new concentrability coefficient, $N_0$ and $N_1$ are the numbers of offline and online samples, respectively. For regret minimization, we show that it achieves a constant $\tilde{O}( \sqrt{N_1/(N_0/\mathtt{C}(\pi^{-}|\rho)+N_1)} )$ speed-up compared to pure online learning, where $\mathtt{C}(\pi^-|\rho)$ is the concentrability coefficient over all sub-optimal policies. Our results also reveal an interesting separation on the desired coverage properties of the offline dataset for sub-optimality gap minimization and regret minimization. We further validate our theoretical findings in several experiments in special RL models such as linear contextual bandits and Markov decision processes (MDPs).
\end{abstract}

\section{Introduction}

Sequential decision making \citep{lattimore2020bandit,Sutton:1998,Bubeck:2012} is often cast as an online learning problem, where an agent interacts with its environment and dynamically updates its policy based on the actions and feedback. The fundamental challenge lies in the exploration-exploitation tradeoff, requiring the agent to balance exploiting known high-reward actions with exploring potentially beneficial but uncertain alternatives. Although exploration is a must-have for sequential decision making, there are unavoidable costs, e.g., performance degradation, incurred by exploration during the online learning process, which are often undesirable in practical applications.

To overcome the drawbacks, offline policy learning has been studied in both bandits \citep{li2022pessimism,wang2023oracle,oetomo2023cutting,zhang2019warm,brandfonbrener2021offline,nguyen2021offline} and reinforcement learning \citep{hester2018deep, nair2018overcoming, nair2020awac, rajeswaran2017learning, lee2022offline, xie2021policy, song2022hybrid, wagenmaker2023leveraging, agrawal2023optimal, li2023rewardagnostic}. In this setting, the agent attempts to learn an optimal policy based solely on an offline dataset that was collected a priori by a behavior policy, without any online interaction with the environment. This setting has attracted growing interest mainly because in many practical applications such as recommendation systems \citep{thomas2017predictive}, healthcare \citep{gottesman2019guidelines}, and wireless networking \citep{yang2023offline}, logged data is often available from prior tasks while acquiring new data is costly. A critical challenge in offline policy learning, however, is that its performance depends critically on the quality of the dataset. 

A natural solution that achieves the benefits of both online and offline settings is \emph{hybrid learning}, where the agent has access to an offline dataset while also having the ability to interact with the environment in an online fashion. 
A number of works \citep{hester2018deep, nair2018overcoming, nair2020awac, rajeswaran2017learning, lee2022offline, song2022hybrid} have empirically demonstrated that offline datasets can help online learning. However, there are only limited studies that theoretically investigate the efficiency of hybrid learning \citep{xie2021policy, song2022hybrid, wagenmaker2023leveraging, agrawal2023optimal, li2023rewardagnostic}. 
It has been shown that in tabular MDPs,  hybrid learning can outperform pure offline RL and pure online RL algorithms in terms of the sample complexity required to identify an $\epsilon$-optimal policy \citep{li2023rewardagnostic}. Similar results have been obtained in linear MDPs~\citep{wagenmaker2022instance} and stochastic $K$-armed bandits~\citep{agrawal2023optimal}. The benefit of utilizing offline datasets to reduce the online learning regret has been characterized for RL with general function approximation in \citet{tan2024natural}. A complete literature review can be found in \Cref{sec:related_works}.

Despite the theoretical successes of hybrid learning in RL, to the best of our knowledge, there is a lack of a unified understanding regarding its benefits. The main question we aim to answer is: 
\begin{center}
\textit{Can we develop a unified algorithm for hybrid RL and characterize the fundamental impact of offline datasets?}
\end{center}

In this work, we provide an affirmative answer to the above question through a generic hybrid RL framework and two lower bounds. In addition, by analyzing the benefits of hybrid RL through a unified analysis for sub-optimality gap and regret, which are two key metrics measuring RL algorithms, our findings also answer the following question affirmatively: 
\begin{center}
\textit{Do we need different offline datasets when we minimize sub-optimality gap and regret in hybrid RL?}
\end{center}
We summarize our contributions as follows. 
\begin{itemize}[noitemsep,topsep=0pt,leftmargin = *]
    \item We first establish a framework based on a commonly adopted notion in decision-making problems, namely the uncertainty level. This framework is used to derive a novel concentrability coefficient $\mathtt{C}(\pi|\rho)$ and to analyze the sub-optimality gap or regret of hybrid RL algorithms. We show that if a confidence-based online RL algorithm is augmented with the offline dataset, the sub-optimality gap scales in the order $\tilde{O}(1/\sqrt{N_0/\mathtt{C}_1 + N_1} )$ and the regret scales in the order $\tilde{O}(N_1/\sqrt{N_0/\mathtt{C}_2 + N_1})$, where $N_0$ is the size of the offline dataset, $N_1$ is the number of episodes during online learning, and $\mathtt{C}_1$ and $\mathtt{C}_2$ are concentrability coefficients. Compared to the sub-optimality gap $\tilde{O}(1/\sqrt{N_1})$ and regret $\tilde{O}(\sqrt{N_1})$ of pure online learning, our results demonstrate a constant $\sqrt{N_1/(N_0/\mathtt{C} + N_1)}$ speed-up compared to pure online learning, where $\mathtt{C}$ is the concentrability coefficient that depends on the problem. We also specialize our general framework and results to linear contextual bandits and Markov decision processes (MDPs) for a better understanding. A full comparison of our results with existing results in the literature is provided in \Cref{table}.
    
    \item Then, we derive lower bounds for both sub-optimality gap and regret minimization problems. Specifically, we show that any hybrid RL algorithm must incur a sub-optimality gap that scales in $\tilde{\Omega}(1/\sqrt{N_0/\mathtt{C}_1 + N_1})$  and regret that scales in $\tilde{\Omega}(N_1/\sqrt{N_0/\mathtt{C}_2 + N_1})$. These results show that initializing with an offline dataset as in the proposed hybrid RL framework is order-wisely optimal in terms of the number of samples and concentrability coefficient.
    
    \item Our upper bound reveals the fundamental impact of the behavior policy used to collect the offline dataset on the performance of hybrid learning. In particular, for {\it sub-optimality gap minimization}, our results show that if the behavior policy has good coverage on the {\it optimal policy}, 
    the sub-optimality gap of hybrid learning can be very low. On the other hand, for {\it regret minimization}, as long as the behavior policy provides good coverage on {\it any sub-optimal policies}, hybrid learning can help reduce the regret. Such separation highlights the fundamental distinction between those two performance metrics, and invites further investigation in the hybrid learning setting.

    \item Finally, we validate our findings in classic MDP examples such as linear contextual bandits and tabular MDPs. The empirical results verify the theoretical benefit of hybrid learning and the impact of different offline behavior policies, particular the aforementioned separation performance of offline data under sub-optimality gap and regret minimization problems.
\end{itemize}

\begin{table*}[h]
\caption{\small Comparison of results on sub-optimality gap (SOG) and regret. All results omit the big-O notation and logarithm terms. $\mathtt{C}_w$ is an all-policy concentrability coefficient (CE). $\mathtt{C}_{\mathrm{off}}$ and $\mathtt{C}_{\mathrm{on}}$ are inversely related coefficients. $\mathtt{C}_{l}$ is a single-policy CE. $\mathtt{C}$ is a CE defined in multi-armed bandits. $\mathtt{C}(\pi|\rho)$ is defined in \Cref{def:Coverbility}. In the 4-th row, offline indicates pure offline learning algorithms \citep{rashidinejad2021bridging, xie2021policy} and online indicates pure online algorithms \citep{lattimore2020bandit,Sutton:1998}. In the 5-th row, the orders apply to both lower and upper bounds. Our results match or outperform the SOTA and show clear differences between SOG and regret minimization. 
}    \label{table}
\begin{center}
\begin{tabular}{lc|lc}
\Xhline{3\arrayrulewidth}
\textbf{Algorithm}  &\textbf{Sub-optimality Gap} & \textbf{Algorithm}  & \textbf{Regret} \\
\hline
FTPEDAL      & \multirow{2}{*}{ $ \frac{1}{\sqrt{ N_0/\mathtt{C}_w(\cdot|\rho) +N_1 }}  $  } & DISC-GOLF & \multirow{2}{*}{$  \sqrt{N_1}\sqrt{\frac{\mathtt{C}_{\mathrm{off}}N_1}{N_0}} + \sqrt{\mathtt{C}_{\mathrm{on}}N_1} $}\\
\small \citet{wagenmaker2023leveraging} &  &  \citet{tan2024natural} &  \\
\hline
RAFT   & \multirow{2}{*}{ $\sqrt{\frac{\mathtt{C}_{\mathrm{off}}}{N_0+N_1} } + \sqrt{\frac{\mathtt{C}_{\mathrm{on}}}{N_1}}$ } & MIN-UCB & \multirow{2}{*}{$  \frac{N_1}{\sqrt{N_0/\mathtt{C} + N_1}} $ }\\
\small \citet{li2023rewardagnostic} & & \small \citet{cheung2024leveraging} & \\
\hline
Offline \& Online &  \multirow{2}{*}{ $\sqrt{\frac{\mathtt{C}_l}{N_0}} $ \& $ \frac{1}{\sqrt{N_1}} $ } & Online & \multirow{2}{*}{ $ \sqrt{N_1}$ }  \\
\citet{uehara2021representation} & & \cite{jin2020provably} \\
\hline

Ours  &  \multirow{2}{*}{ $\frac{1}{\sqrt{N_0/\mathtt{C}(\pi^*|\rho) + N_1}}$ } & Ours & \multirow{2}{*}{ $\frac{N_1}{\sqrt{N_0/\mathtt{C}(\pi^{-\varepsilon}|\rho) + N_1}} $  } \\
\small Theorems \ref{thm:suboptimality} and \ref{thm:lowerbound} &  & \small Theorems \ref{thm:main regret} and \ref{thm:lowerbound} & \\

\Xhline{3\arrayrulewidth}
\end{tabular}
\end{center}
\end{table*} 

\section{Problem Formulation}\label{sec:pre}

{\bf Notations.} 
Throughout this paper, we use $\|x\|_V$ to denote $\sqrt{x^\TT V x}$. 
The set of all probability distributions over a set \(\mathcal{X}\) is represented by \(\Delta(\mathcal{X})\).
\(\mathbbm{1}\{\cdot\}\) stands for the indicator function, and $[H]=\{1,2,\ldots,H\}$ for $H\in\mathbb{N}$.

\subsection{Preliminaries}

\textbf{Reinforcement Learning.}
We consider episodic Markov decision processes in the form of $\mathcal{M} = (\Xc,\mathcal{A}, P, H, R, q)$, where $\Xc$ is the state space and $\mathcal{A}$ is the action space, $H$ is the number of time steps in each episode, $P=\{P_h\}_{h=1}^H$ is a collection of transition kernels, and $P_h(x_{h+1}|x_h,a_h)$ denotes the transition probability from the state-action pair $(x_h,a_h)$ at step $h$ to state $s_{h+1}$ in the next step, $r=\{r_h\}_{h=1}^H$ is a collection of reward functions of state-action pairs, where $r_h:\Xc\times\mathcal{A}\rightarrow[0, 1]$, $q\in\Delta(\Xc)$ is the initial state distribution. 

A Markov policy $\pi$ is a set of mappings $\{\pi_h : \Xc \rightarrow \Delta(\mathcal{A})\}_{h=1}^H$. In particular, $\pi_h(a|s)$ denotes the probability of selecting action $a$ in state $s$ at time step $h$. We denote the set of all Markov policies by $\Pi$.  For an agent adopting policy $\pi$ in an MDP $\mathcal{M}$, at each step $h \in [H]$, the agent observes state $x_h 
\in \Xc$, and takes an action $a_h \in \mathcal{A}$ according to $\pi$, after which the agent receives a random reward $r_t\in[0,1]$ whose expectation is $r(x_t,a_t)$ and the environment transits to the next state $x_{h+1}$ with probability $P_h(x_{h+1}|x_h,a_h)$. The episode ends after $H$ steps. 

Let $V^\pi_{P}$ be the value function of policy $\pi$ under the transition model $P$. Mathematically, $V^\pi_{\Mc}: =  \mathop{\Eb}\left[\sum_{h=1}^H R_h(x_{h},a_{h})\big|P, \pi, q\right]$, where the expectation is taken over all random variables including reward $R$, state $x_h$ and action $a_h$.

\subsection{Hybrid Learning}
Hybrid learning seeks to combine the advantages of both online and offline learning.
Specifically, the learning agent has access to a finite offline dataset $\mathcal{D}_0 \subset (\Xc)^H\times (\mathcal{A})^H\times [0,1]^H $ with size $N_0$. Each data point $\tau\in\Dc_0$ has the form $\tau=(x_1,a_1,r_1,\ldots,x_H,a_H,r_H)$, also called a trajectory, and is randomly sampled under a behavior policy $\rho$ from the ground-truth MDP environment $\mathcal{M}^*=(\Xc,\mathcal{A},P^*,H,R^*,q^*)$.
Then, the agent performs online learning with the knowledge of offline data. Let $\pi_t$ be the policy chosen by the agent at episode $t\geq1$. Denote $\tau_t\sim\pi_t$ as the trajectory sampled from $\pi_t$. Let $\Dc_{t} = \{\tau_t\}_{s=1}^{t}$ be the data collected in the online learning procedure after episode $t$. In this paper, we consider two classical learning objectives, as elaborated below.

\textbf{Sub-optimality Gap Minimization.} 
For this learning objective, the goal of the agent is to learn a policy $\hat{\pi}$ from both the offline dataset $\Dc_0$ and the online dataset $\Dc_{N_1}$ such that the sub-optimality gap of $\hat{\pi}$, defined in \Cref{eqn:subopt gap}, is minimized. 
\begin{align}
    \text{Sub-opt}(\hat{\pi}) &= \max_{\pi\in\Pi} V_{\Mc^*}^{\pi} - V_{\Mc^*}^{\hat{\pi}}.\label{eqn:subopt gap}
\end{align}
We remark that this objective is widely studied in both online and offline RL literature ~\citep{li2022pessimism,uehara2021representation,jin2021pessimism}. 

\textbf{Regret Minimization.} For this learning objective, the agent aims to minimize the regret  during the online interactions with horizon $N_1$, as defined below:
\begin{align}
    \text{Regret}(N_1) = \sum_{t=1}^{N_1} \left( \max_{\pi\in\Pi} V_{\Mc^*}^{\pi} - V_{\Mc^*}^{\pi_t} \right). \label{eqn:regret}
\end{align}
Regret minimization has also been studied intensively~\citep{abbasi2011improved,lattimore2020bandit,sharma2020warm,silva2023user,shivaswamy2012multi}.

\section{A Unified Hybrid RL Framework}\label{sec:unify alg}
In this section, we first present a unified framework for hybrid RL, and then analyze its performance under certain general assumptions. We would like to emphasize that both the learning framework and the analysis are quite universal and can be applied to various MDP settings. 


\subsection{A Unified Hybrid RL Framework}\label{sec:alg design} 

\textbf{Oracle Algorithm.}
The core of the hybrid RL framework relies on an oracle algorithm, denoted as $\mathtt{Alg}$. $\mathtt{Alg}$ takes a dataset $\Dc$ sampled from an unknown environment $\Mc^*$ as its input, and is able to output: (i) an 
estimator that estimates the value function $V_{\Mc^*}^{\pi}$ for any policy $\pi$, denoted as $\hat{V}_{\mathtt{Alg}}^{\pi}$; and (ii) an uncertainty function $\hat{\mathtt{U}}^\pi_{\mathtt{Alg}}$ that upper bounds the estimation error in $\hat{V}_{\mathtt{Alg}}^{\pi}$ with high probability, i.e.,
\[\hat{\mathtt{U}}_{\mathtt{Alg}}^\pi \geq \left| V_{\Mc^*}^{\pi} - \hat{V}_{\mathtt{Alg}}^{\pi}\right|\] 
with probability at least $1-\delta$ for $\delta\in(0,1)$. 

In the following, we use $(\hat{V}_{\mathtt{Alg}}^{\pi}(\Dc), \hat{\mathtt{U}}_{\mathtt{Alg}}^{\pi}(\Dc))$ to denote the output of $\mathtt{Alg}$ for a given input $\Dc$. We sometimes omit $\Dc$ from the notation when it is clear from the context.


\if{0}
Confidence-based methods \citep{abbasi2011improved,nguyen2021offline,uehara2021representation} leverage the agent's confidence in its knowledge about the environment to balance the exploration-exploitation trade-off more effectively, improving the learning efficiency and performance. One critical enabler for such methods is an accurate measure of the uncertainty level. One celebrated example is the confidence bound utilized in the upper confidence bound (UCB) algorithm~\citep{abbasi2011improved,jin2020provably,azar2017minimax}. We formalize the confidence-based algorithm into the following definition.



\begin{definition}[Confidence-based Algorithm]
Fix an error probability $\delta\in(0,1)$. A confidence-based algorithm $\mathtt{Alg}$ can complete the following tasks: For any MDP $\Mc^*$, given a dataset $\Dc$ sampled from $\Mc^*$, $\mathtt{Alg}$ can (a) estimate the value function $V_{\Mc^*}^{\pi}$ for any policy $\pi$, denoted as $\hat{V}_{\mathtt{Alg}}^{\pi}$; (b) compute an uncertainty function $\hat{\mathtt{U}}_{\mathtt{Alg}}(\pi|\Dc)$ which satisfies that, with probability at least $1-\delta$,
\[\hat{\mathtt{U}}_{\mathtt{Alg}}(\pi|\Dc) \geq \left| V_{\Mc^*}^{\pi} - \hat{V}_{\mathtt{Alg}}^{\pi}\right| ,\] where the randomness comes from data samples $\Dc$; (c) at each episode, generate $\pi_t=\arg\max_\pi\{\hat{\mathtt{U}}_{\mathtt{Alg}}(\pi|\Dc)\}$ in sub-optimality gap minimization problem and $\pi_t = \arg\max_\pi\{\hat{V}_{\mathtt{Alg}}^{\pi} + \hat{\mathtt{U}}_{\mathtt{Alg}}(\pi|\Dc)\}$ in regret minimization problem. 
\end{definition}

\fi

\textbf{The Unified Hybrid RL Framework.} With the pre-selected oracle algorithm $\mathtt{Alg}$, we are ready to present the unified hybrid RL framework. 

Specifically, at each online episode $t\in [N_1]$, we maintain an online dataset $\Dc_{t-1}$, which stores all trajectories collected during the online learning so far. Instead of using $\Dc_{t-1}$ to find the next online policy, we augment $\Dc_{t-1}$ with the offline dataset $\Dc_0$ and feed $\Dc_0\cup\Dc_{t-1}$ to the oracle algorithm $\mathtt{Alg}$. With the output $(\hat{V}_{\mathtt{Alg}}^{\pi}, \hat{\mathtt{U}}_{\mathtt{Alg}}^{\pi})$, we then construct the online policy $\pi_t$ following the optimism in face of uncertainty principle. I.e., we set $\pi_t$ to be the policy that maximize the upper confidence bound (UCB) of the expected return defined as $\hat{V}_{\mathtt{Alg}}^{\pi} +\hat{\mathtt{U}}_{\mathtt{Alg}}^\pi$. We then
collect the new trajectory $\tau_t = (x_{t,1},a_{t,1},r_{t,1},\ldots, x_{t,H},a_{t,H},r_{t,H})$ and update $\Dc_t = \Dc_{t-1}\cup\{\tau_t\}$. Note that the \textbf{regret} during online learning is exactly the summation of the sub-optimality gaps of $\{\pi_t\}_{t\in[N_1]}$.

For the learning objective of  \textbf{sub-optimality gap minimization}, the agent will need to output a near-optimal policy at the end of online learning phase. The policy is then obtained by utilizing the well-known pessimism principle. Specifically, the agent feeds $\Dc_0\cup\Dc_{N_1}$ to $\mathtt{Alg}$ and obtains  $(\hat{V}_{\mathtt{Alg}}^{\pi}, \hat{\mathtt{U}}_{\mathtt{Alg}}^{\pi})$. Then, the lower confidence bound (LCB) of the expected return can be expressed as $\hat{V}_{\mathtt{Alg}}^{\pi}- \hat{\mathtt{U}}_{\mathtt{Alg}}^{\pi}$, and the near-optimal policy is the one that maximizes the LCB.
The pseudo-code is presented in \Cref{alg:hybrid}. 


\begin{algorithm}[htbp]
    \caption{Hybrid RL Framework}
    \label{alg:hybrid}
    \begin{algorithmic}[1]
        \STATE {\bf Input:} Offline dataset $\Dc_0$, total online steps $N_1$.
        \FOR{$t=1,\ldots, N_1$}
        \STATE $(\hat{V}_{\mathtt{Alg}}^{\pi}, \hat{\mathtt{U}}_{\mathtt{Alg}}^{\pi})\gets\mathtt{Alg}(\Dc_0\cup\Dc_{t-1})$
            \STATE Execute policy $\pi_t=\arg\max_\pi \hat{V}_{\mathtt{Alg}}^{\pi} +\hat{\mathtt{U}}_{\mathtt{Alg}}^\pi$. \STATE Collect trajectory $\tau_t$. 
            \STATE Update $\Dc_t=\Dc_{t-1}\cup\{\tau_t\}$.
        \ENDFOR
        \IF{sub-optimality gap minimization:}
        \STATE $(\hat{V}_{\mathtt{Alg}}^{\pi}, \hat{\mathtt{U}}_{\mathtt{Alg}}^{\pi})\gets\mathtt{Alg}(\Dc_0\cup\Dc_{N_1})$.
        \STATE {\bf Output:} $\hat{\pi} = \arg\max_\pi \hat{V}_{\mathtt{Alg}}^{\pi} -\hat{\mathtt{U}}_{\mathtt{Alg}}^\pi$. 
        \ENDIF
    \end{algorithmic}
\end{algorithm}

We remark that \Cref{alg:hybrid} enjoys a clean structure where we can utilize the online confidence-based algorithms by simply augmenting with offline data. Such an approach is more practically amenable compared with the much more complicated hybrid RL algorithm design in \citet{li2023rewardagnostic,wagenmaker2023leveraging,tan2024hybrid}. 

\subsection{Theoretical Analysis}\label{sec:thm sample}
In this section, we analyze the theoretical performance of the unified hybrid RL framework presented in \Cref{alg:hybrid}. Intuitively, the quality of the offline dataset $\Dc_0$ is of paramount importance for the hybrid learning performance. 
To assess the quality of the behavior policy $\rho$ and the offline dataset $\Dc_0$, we first introduce several key concepts and properties, including concentrability coefficient, and Eluder-type condition.

\begin{definition}[Uncertainty level]\label{def:uncertainty}
Let $\mathtt{Alg}_0$ be the best oracle algorithm that achieves the minimum estimation error in the worst case, i.e.,
\[ \mathtt{Alg}_0 = \arg\min_{ \mathtt{Alg} } \max_{\Mc} \mathbb{E}_{\Dc_0\sim(\Mc,\rho)} \left[ \left| V_{\Mc}^{\pi} - \hat{V}_{\mathtt{Alg}}^{\pi} \right| \right]. \]
 The uncertainty level of a policy $\pi$, denoted as $\mathtt{U}_{\Mc^*}(\pi): \Pi\rightarrow\Rb$, is defined by
\(\mathtt{U}_{\Mc^*}(\pi) = \mathbb{E}_{\Dc_0\sim(\Mc^*,\rho)} \left[ \left| V_{\Mc^*}^{\pi} - \hat{V}_{\mathtt{Alg}_0}^{\pi} \right|\right]. \)
\end{definition}
The reason that we choose a minimax type definition is that $\Mc^*$ is unknown, and any learning algorithm should prepare for the worst case. Moreover, $\mathtt{U}_{\Mc^*}(\pi)$ serves as a lower bound of $\Eb_{\Dc_0\sim(\Mc^*,\rho)}[\hat{\mathtt{U}}_{\mathtt{Alg}}^\pi(\Dc_0)]$. Thus, $\mathtt{U}_{\Mc^*}(\pi)$ is algorithm-independent and represents the essential hardness of estimating the value  $V_{\Mc^*}^{\pi}$. 


 
\begin{definition}[Concentrability coefficient] \label{def:Coverbility} Given a behavior policy $\rho$, the concentrability coefficient of a target policy $\pi$ is
    \( \mathtt{C}(\pi|\rho) = \left( \nicefrac{\mathtt{U}_{\Mc^*}(\pi)}{ \mathtt{U}_{\Mc^*}(\rho) } \right)^2 \in [1, \infty]. \)

\end{definition}

Intuitively, the definition describes how much more effort is needed to estimate $V_{\Mc^*}^{\pi}$ compared to estimating $V_{\Mc^*}^{\rho}$ from an offline dataset of size $N_0$ sampled under $\rho$. 

To enable efficient learning, it is necessary to impose certain conditions on the oracle algorithm $\mathtt{Alg}$. We adopt an
Eluder-type condition, defined as follows. As we will show later, this condition plays a crucial role in controlling exploration and exploitation.

\begin{definition}[Eluder-type condition]\label{eluder condition} Let $N_1$ be the total number of episodes during online learning. Fix an error probability $\delta$. Let $\pi_t$ and $\Dc_{t-1}$ be the policies and dataset generated by an oracle algorithm $\mathtt{Alg}$ at episode $t$. We say $\mathtt{Alg}$ satisfy Eluder-type condition if, with probability at least $1-\delta$, 
     \( \sum_{t=1}^{N_1} \hat{\mathtt{U}}_{\mathtt{Alg}}^{\pi_t}(\Dc_{t-1})^2 \leq C_{\mathtt{Alg}}^2. \)
\end{definition}

At a high level, eluder-type condition is akin to the pigeonhole principle and the elliptical potential lemma
widely used in tabular MDPs \citep{azar2017minimax,menard2021fast} and linear bandits/MDPs \citep{abbasi2011improved,jin2020provably}, respectively. 
Intuitively, $C_{\text{alg}}$ thus depends on the complexity of estimating $V^\pi$ for all encountered policies and can be explicitly computed or upper-bounded in certain classes of RL problems, such as tabular MDPs or linear bandits. The constant exists for most theoretical online reinforcement learning algorithms. As proved in \Cref{sec:example proof}, this constant depends on the design of the algorithm and the complexity of the environment. We will show that the eluder-type condition holds for the specific RL algorithms considered in this work (See \Cref{sec:example proof}).

We further assume that with probability at least $1-\delta$, $\hat{\mathtt{U}}_{\mathtt{Alg}}^\pi(\Dc_0) \leq C_{\mathtt{Alg}} \mathtt{U}_{\Mc}(\pi)$ holds for any $\Mc$ and $\Dc_0$. This is a reasonable assumption, since $\hat{\mathtt{U}}_{\mathtt{Alg}}^\pi(\Dc_0) = O(1)$ and $\max_{\Mc}\mathtt{U}_{\Mc}(\pi)$ has a lower bound. In \Cref{sec:example proof}, we also show how to find $C_{\mathtt{Alg}}$.

\begin{theorem}
    \label{thm:suboptimality}
    Let $\mathtt{Alg}$ satisfy the condition \Cref{eluder condition}, $\hat{\pi}$ be the output policy of \Cref{alg:hybrid}. Suppose $\pi^*$ is an optimal policy. Then, with probability at least $1-O(\delta)$, the sub-optimality gap $\hat{\pi}$ is 
    \[
        \text{Sub-opt}(\hat{\pi}) = \tilde{O}\left(   \frac{ C_{\mathtt{Alg}} }{ \sqrt{N_0/\mathtt{C}(\pi^*|\rho)  + N_1}  } \right),
    \]
    where $N_0$ and $N_1$ are the number of offline and online trajectories, respectively, $\mathtt{C}(\pi^*|\rho)$ is the concentrability coefficient, and $C_{\mathtt{Alg}}$ is defined in \Cref{eluder condition}.

\end{theorem}

\begin{remark}
We elaborate on the performance of the hybrid RL framework with respect to different qualities of the behavior policy $\rho$ and offline data as follows.
\vspace{-0.1in}
\begin{itemize}[leftmargin=*]\itemsep=0pt
    \item When the behavior policy $\rho$ is an optimal policy, we have $\mathtt{C}(\pi^*|\rho) = 1$. The sub-optimality gap of $\hat{\pi}$ is $\tilde{O}(\sqrt{1/(N_0+N_1)})$. This \textit{strictly improves both pure online and offline learning algorithms} where the sub-optimal gap scales in $\tilde{O}(\sqrt{1/N_1})$ and $\tilde{O}(\sqrt{1/N_0})$, respectively. 
    \item When the behavior policy $\rho$ is extremely bad such that $\mathtt{C}(\pi^*|\rho) =\Omega(N_0)$, \Cref{thm:suboptimality} states that the sub-optimality gap of $\hat{\pi}$ is $\tilde{O}(\sqrt{1/N_1})$, which \textit{recovers the optimal pure online learning result}. 
    \item When the behavior policy $\rho$ has partial coverage on the optimal policy, we have $\mathtt{C}(\pi^*|\rho)\in(1,N_0)$. \Cref{thm:suboptimality} suggests that \Cref{alg:hybrid} is \textit{equivalent to an online algorithm with $N_0/\mathtt{C}(\pi^*|\rho) + N_1$ episodes}, while it only runs $N_1$ episodes. Essentially, $N_0/\mathtt{C}^{\gamma}(\pi^*|\rho)$ serves as the number of effective episodes from the offline data.
    
\end{itemize}

\end{remark}
 
\begin{theorem}\label{thm:main regret}
  Let $\mathtt{Alg}$ satisfy the conditions   
  in \Cref{def:uncertainty} and \Cref{eluder condition}. Then, the regret of  \Cref{alg:hybrid} scales as
    \[
        \text{Regret}(N_1) = \tilde{O}\left( C_{\mathtt{Alg}} \sqrt{N_1}\sqrt{  \frac{ N_1 }{ N_0/\mathtt{C}(\pi^{-\varepsilon}|\rho)  + N_1}  } \right),
    \]
    where $\mathtt{C}(\pi^{-\varepsilon}|\rho)$ is the maximum concentrability coefficient of the sub-optimal policies whose sub-optimality gap is at least $\varepsilon$, and $\varepsilon=\tilde{O}(1/\sqrt{N_0 + N_1})$.
\end{theorem}

\begin{remark}
The key observation from \Cref{thm:main regret} is that the regret does {\bf not} depend on the concentrability coefficient over {\bf the optimal policy $\pi^*$}. Rather, it depends on the concentrability coefficient over {\bf sub-optimal policies}, which is \textit{in contrast to the case in sub-optimality gap minimization problem}. 
Thus, a behavior policy that achieves the best sub-optimality gap may lead to poor performance for regret minimization. This phenomenon is confirmed by our experimental results (See \Cref{sec:main exp}).

Specifically, when the behavior policy $\rho$ is an {\it optimal policy}, and the support of optimal policy does not overlap with sub-optimal policies, i.e. $ \mathtt{C}(\pi^*|\rho) = 1$, but $\mathtt{C}(\pi^{-\varepsilon}|\rho)=0$,  \Cref{thm:main regret} suggests that the regret is $\tilde{O}(\sqrt{N_1})$, which {\it recovers the pure online learning result}. While the result seems surprising, it reflects the essential challenge of regret minimization: exploration-exploitation tradeoff. Because the offline policy $\rho=\pi^*$ encodes little exploration information, the agent still needs to explore sub-optimal policies to ensure there is no better policy. This procedure incurs the same regret as pure online learning. One may ask why we cannot use imitation learning in such case. It is because, we {\bf do not} know if the offline policy is the best or not. Our goal is to develop a universal algorithm that is guaranteed to have sub-linear regret in {\bf any} case and thus imitation learning would fail.

On the other hand, when the behavior policy $\rho$ is exploratory such that $\mathtt{C}(\pi^{-\varepsilon}|\rho) = O(1)$, \Cref{thm:main regret} states that the regret is $\tilde{O}(\sqrt{N_1}\sqrt{N_1/(N_0 + N_1)})$, which {\it significantly improves the pure online learning by a factor of $\sqrt{N_1/(N_0 + N_1)}$}. 

Finally, in all cases, \Cref{thm:main regret} proves that \Cref{alg:hybrid} achieves a constant $\Theta(\sqrt{N_1/(N_0/\mathtt{C}(\pi^{-\varepsilon}|\rho)+N_1)})$ speed-up compared with pure online learning. 
\end{remark}

\section{Examples}\label{sec:eg}
In this section, we specialize \Cref{alg:hybrid} to two classic examples, namely, tabular MDPs and linear contextual bandits, by specifying the corresponding oracle algorithm $\mathtt{Alg}$ to obtain the estimator of the value function and the uncertainty function.

\if{0}
Existing work use specific $\mathtt{Alg}$ to define concentrability coefficient. For example, in tabular MDPs, the commonly adopted concentrability coefficient $C_{cov}$ is calculated as follows.
\[C_{cov} = \max_h \max_{s_h,a_h} \nicefrac{\mathbb{P}_{\Mc^*}^{\pi}(s_h,a_h)}{\mathbb{P}_{\Mc^*}^{\rho}(s_h,a_h)},\]
where $\mathbb{P}_{\Mc}^{\pi}(s_h,a_h)$ is the marginal probability of state action pair $(s_h,a_h)$ under MDP $\Mc$ and policy $\pi$~\citep{li2023rewardagnostic,xie2023the}.
On the other hand, the estimation error of $\hat{V}_{\mathtt{ALg}}^\pi$ can be upper bounded can be calculated as 
\begin{align*}
        \left|V_{\Mc^*}^{\pi} - V_{\mathtt{Alg}}^{\pi}\right| \leq \Eb\left[\sum_{h=1}^H \beta/\sqrt{N_h(x_h,a_h)} \middle| \Mc^*, \pi\right].
    \end{align*}
We note that the difference between the RHS of the above inequality and \Cref{eqn: tabular MDP bonus informal} is the environment that are conditioned on. Then, due to the fact that $\frac{\sum_i \alpha_i x_i}{\sum_{i}\beta_i x_i} \leq \max_i \frac{\alpha_i}{\beta_i}$ when $\alpha_i,\beta_i,x_i>0$, we have 
\begin{align*}
     \nicefrac{\left|V_{\Mc^*}^{\pi} - V_{\mathtt{Alg}}^{\pi}\right|}{ \left|V_{\Mc^*}^{\rho} - V_{\mathtt{Alg}}^{\rho}\right|} \leq C_{cov}.
\end{align*}
This motivates us to define an algorithm-independent concentrability coefficient, which captures the essential quality of the behavior policy. 
\fi

\subsection{Tabular MDPs}
Tabular MDPs assume that the state and action spaces are finite. Provided a dataset $\Dc = \{\tau_t\} $, where $\tau_t=(x_{t,1},a_{t,1},\ldots, x_{t,H}, a_{t,H})$ is a trajectory sampled from $\Mc^*$, a classic method to estimate the reward and transition kernel is as follows:
    \begin{align}
    \left\{
        \begin{aligned}
        &\hat{r}_h(x_h,a_h) = \frac{ \sum_{t} \mathbbm{1}\{(x_{t,h},a_{t,h}) = (x_h,a_h)\} r_{t,h}  }{ N_h(x_h,a_h) },\\
        &\hat{P}_h(x_{h+1}|x_h,a_h) = \frac{ N_h(x_{h+1}, x_h, a_h) }{ N_{h}(x_h,a_h) },
        \end{aligned}
    \right.\label{eqn:UCB-VI}
    \end{align}
    where $N_h(x_h,a_h) = \sum_{t}\mathbbm{1}\{(x_{t,h},a_{t,h}) = (x_h,a_h) \}$ and $N_h(x_{h+1}, x_h, a_h) = \sum_t\mathbbm{1}\{(x_{t,h+1}, x_{t,h},a_{t,h}) = (x_{h+1}, x_h,a_h)\}$. \citet{azar2017minimax} has shown that the estimated model $\hat{\Mc}= \{\Xc,\mathcal{A},\hat{P},H,\hat{r},\hat{q}\}$ satisfies
    \begin{align}
        \left|V_{\Mc^*}^{\pi} - V_{\hat{\Mc}}^{\pi}\right| \leq \Eb\left[\sum_{h=1}^H \beta/\sqrt{N_h(x_h,a_h)} \middle| \hat{\Mc}, \pi\right], \label{eqn: tabular MDP bonus informal}
    \end{align}
    for some $\beta = \tilde{O}(H)$. Thus, we can use the RHS of \Cref{eqn: tabular MDP bonus informal} as the uncertainty function $\hat{\mathtt{U}}^{\pi}_{\mathtt{Alg}}$ and $V_{\hat{\Mc}}^{\pi}$ as the estimated value function.
More importantly, the uncertainty function satisfies the eluder-type condition. Then, we have the following result.
    
    \begin{corollary}\label{coro:tbl mdp}
    For tabular MDPs, under the hybrid RL framework in \Cref{alg:hybrid}, using $\hat{\mathtt{U}}^{\pi}_{\mathtt{Alg}}$ defined in the RHS of \Cref{eqn: tabular MDP bonus informal}, the regret scales in 
    \[ \tilde{O}\left(\sqrt{H^4|\Xc||\mathcal{A}|N_1} \sqrt{\frac{N_1}{N_0/\mathtt{C}(\pi^{-\varepsilon}|\rho) + N_1}} \right); \]
    and the sub-optimality gap is 
    \[ \tilde{O}\left( \sqrt{ \frac{H^4|\Xc||\mathcal{A}|}{ N_0/\mathtt{C}(\pi^*|\rho) + N_1 } } \right).  \]
    \end{corollary}

    Next, we analyze our concentrability coefficient in tabular MDPs. It is shown \citep{azar2017minimax} that using the RHS of \Cref{eqn: tabular MDP bonus informal} $\hat{\mathtt{U}}_{\mathtt{Alg}}$ achieves the optimal order of $|\Xc||\mathcal{A}|$ in learning regret, in the following, we use $\hat{\mathtt{U}}_{\mathtt{Alg}}$ as a proxy of the uncertainty level $\mathtt{U}_{\Mc^*}(\pi) $ defined in \Cref{def:uncertainty}. 
    
    Then, by defining the occupancy measure at step $h$ as
\( 
    d_h^{\pi}(x,a) = \Eb[\mathbbm{1}\{x_h=x, a_h=a\}|\Mc^*,\pi],
\)
we have 
\begin{align*}
    \sqrt{\mathtt{C}(\pi|\rho)} &\approx \frac{ \sum_{h}\sum_{x_h,a_h} d_h^{\pi}(x_h,a_h) \frac{1}{\sqrt{N_h(x_h,a_h)}} }{ \sum_{h}\sum_{x_h,a_h}d_h^{\rho}(x_h,a_h)\frac{1}{\sqrt{N_h(x_h,a_h)}} } \\
    &\leq \max_{h,x_h,a_h} \frac{ d_h^{\pi}(x_h,a_h) }{ d_h^{\rho}(x_h,a_h) },
\end{align*}
where the RHS is widely adopted concentrability coefficient in tabular MDPs \citep{xie2021policy,li2024settling}. This inequality indicates that the concentrability coefficient defined in \cref{def:Coverbility} is lower than the existing definition for tabular MDPs. Thus, our upper bounds are tighter than the existing results~\citep{xie2021policy,tan2024hybrid}.

{We remark that existing works~\citep{li2023rewardagnostic,tan2024natural,tan2024hybrid} typically show a sub-optimality gap scales in $\sqrt{\mathtt{C}_{\mathrm{off}}/(N_0+N_1)} + \sqrt{\mathtt{C}_{\mathrm{on}}/N_1}$ and a regret scales in $\sqrt{N_1}\left(\sqrt{\mathtt{C}_{\mathrm{off}}N_1/N_0} + \sqrt{\mathtt{C}_{\mathrm{on}}N_1}\right)$ where $\mathtt{C}_{\mathrm{off}}$ and $\mathtt{C}_{\mathrm{on}}$ are concentrability coefficients for separate state-action spaces (e.g. $\mathtt{C}_{\mathrm{on}} = \max_h\max_{(s_h,a_h)\in G} \frac{d_h^{\pi}(s_h,a_h)}{d_h^{\rho}(s_h,a_h)}$ for some set $G$). Besides it is hard to find the exact value of $\mathtt{C}_{\mathrm{off}}, \mathtt{C}_{\mathrm{on}}$, these results can only match with ours (otherwise are higher than ours) under a strict condition $\mathtt{C}_{\mathrm{on}} = O(N_1/(N_0+N_1))$ and $\mathtt{C}_{\mathrm{off}} = O(N_0/(N_0+N_1))$, which is a rare occurrence. Hence, our results are tighter and easier to interpret. }

\subsection{Linear Contextual Bandits}
Linear contextual bandits is a special case of MDPs when $H=1$, and the reward admits a linear structure. While it simplifies the transition kernel, the linearity captures a core structure in many complex MDPs such as linear MDPs~\citep{jin2020provably} and low-rank MDPs~\citep{uehara2021representation}. Specifically, each state-arm or context-arm pair $(x,a)\in \Xc \times \mathcal{A}$ is associated with a feature vector $\phi(x,a)\in\Rb^d$. At episode $t$, the learning agent observes a context $x_t$ sampled from $q^*$ and then pulls an arm $a_t$. By doing so, the agent receives a reward $r_t=\phi(x_t,a_t)^\TT \theta^*+\xi_t$, where $\xi_t$ is a random noise and $\theta^*\in \Rb^d$ is an unknown parameter. 
Throughout the paper, we assume that $\|\theta^*\|_2 \le 1$ and $\|\phi(x,a)\|_2\le 1$, $\forall (x,a)\in\Xc\times\cA$. 
We also assume that $\xi_t$ is an independent zero-mean sub-Gaussian noise with parameter 1, i.e, $\Eb[\exp(\lambda \xi_t)]\leq \exp(\lambda^2/2)$. 

Many classic algorithms of linear contextual bandits usually involve estimating the unknown parameter $\theta^*$ based on available data $\Dc:=\{(x_t, a_t,r_t)\}_t$ through the linear regression defined as follows~\citep{abbasi2011improved,lattimore2020bandit}: 
\begin{align}\label{eqn:ridge}
    \hat{\theta}= \arg\min_{\theta} \sum_{(a_t,r_t)\in \Dc} (\phi(x_t,a_t)^\TT \theta-r_t)^2+\lambda \|\theta\|^2_2,
\end{align}
where $\lambda>0$ is a given parameter. 
Let $\hat{\Lambda} := \lambda I_d + \sum_{(x_t,a_t)\in \Dc} \phi(x_t, a_t)\phi(x_t,a_t)^\TT$. Then, the solution to \Cref{eqn:ridge} can be expressed as \(
\hat{\theta} = \hat{\Lambda}^{-1} \sum_{(x_t, a_t,r_t)\in \Dc} r_t\phi(x_t, a_t).\)
Furthermore, by choosing $\lambda=d$, with high probability, the following inequalities hold for any $x,a$ (See \citet{abbasi2011improved}):
\begin{align}
    & |\phi(x,a)^\TT \hat{\theta} -  \phi(x,a)^\TT \theta^*| \leq \beta \|\phi(x,a)\|_{\hat{\Lambda}^{-1}},\label{eqn:lcb ucb}
\end{align}
where $\beta=\tilde{O}(\sqrt{d})$. Therefore, we can use $\Eb_{x\sim q^*, a\sim\pi(\cdot|x)}[\phi(x,a)^\TT\hat{\theta}]$ as an estimated value function, and the RHS of \Cref{eqn:lcb ucb} as the uncertainty function $\hat{\mathtt{U}}_{\mathtt{Alg}}$.

\textbf{Linear contextual bandits.} 
If we use $\hat{V}^{\pi} = \Eb_{x\sim q^*, a\sim\pi(\cdot|x)}[\phi(x,a)^\TT\hat{\theta}]$ as the estimator, and the RHS of \Cref{eqn:lcb ucb} as the uncertainty function $\hat{\mathtt{U}}_{\mathtt{Alg}}$, the corresponding algorithm is known as Lin-UCB~\citep{abbasi2011improved}, which satisfies the Eluder-type condition. Then, we have the following result.

\begin{corollary}\label{coro:LinConBandits}
    For linear contextual bandits, under the hybrid RL framework in \Cref{alg:hybrid}, using $\hat{\mathtt{U}}_{\mathtt{Alg}}$ as defined in \Cref{eqn:lcb ucb}, the regret is
    \[ \tilde{O}\left(d\sqrt{N_1}\sqrt{\frac{N_1}{N_0/\mathtt{C}(\pi^{-\varepsilon}|\rho) + N_1}} \right);\]
    and the sub-optimality gap is
    \[ \tilde{O}\left( d\sqrt{\frac{1}{N_0/\mathtt{C}(\pi^{-\varepsilon}|\rho) + N_1} }\right) .\]
\end{corollary}

Since Lin-UCB is shown to be nearly minimax optimal~\citep{chu2011contextual,he2022reduction}, we can use $\hat{\mathtt{U}}_{\mathtt{Alg}}$ as an approximate of the uncertainty level:
$\mathtt{U}_{\Mc^*}(\pi) \approx \beta  \|\Eb_{x\sim q^*} \Eb_{a\sim\pi(x)}[\phi(x,a)]\|_{\hat{\Lambda}_0^{-1}} $. 
Therefore,  we have 
\begin{align*}
    \mathtt{C}(\pi|\rho) &\approx \frac{ \|\Eb_{x\sim q^*} \Eb_{a\sim\pi(x)}[\phi(x,a)]\|^2_{\hat{\Lambda}_0^{-1}} }{ \|\Eb_{x\sim q^*} \Eb_{a\sim\rho(x)}[\phi(x,a)]\|^2_{\hat{\Lambda}_0^{-1}}  } \\
    &\leq \max_{w:\|w\|=1} \frac{ w^\top \Eb_{x\sim q^*} \Eb_{a\sim\pi(x)}[\phi(x,a)\phi(x,a)^\TT] w }{ w^\top \Eb_{x\sim q^*} \Eb_{a\sim\rho(x)}[\phi(x,a)\phi(x,a)^\TT] w  },
\end{align*}
where the RHS is widely adopted as a concentrability coefficient in linear MDPs or low-rank MDPs~\citep{uehara2021representation,tan2024natural}. This inequality indicates that our result is tighter than existing results, especially for problems with a linear structure.



We note that in multi-armed bandits, our regret can be further simplified to $\tilde{O}(N_1/\sqrt{N_0\min_a\rho(a) + N_1})$, which matches the result in \citet{cheung2024leveraging} order-wisely.

\section{Lower Bounds}\label{sec:lower bound}
In this section, we provide a lower bound for hybrid RL. The lower bound shows the tightness of \Cref{thm:suboptimality} and \Cref{thm:main regret}.

\begin{theorem}\label{thm:lowerbound}
    There exists an MDP instance such that any hybrid RL algorithm must incur a sub-optimality gap in 
    \(\Omega \left( \frac{1}{\sqrt{N_0/\mathtt{C}(\pi^*|\rho) + N_1}} \right),\)
    and regret in 
    \(
    \Omega\left( \frac{N_1}{\sqrt{N_0/\mathtt{C}(\pi^{-\varepsilon}|\rho) + N_1}}\right).\)
\end{theorem}

\if{0}
\begin{remark}
    We remark that the regret lower bound is not a contradiction with \Cref{thm:main regret}. Note that our instance is a linear contextual bandit with uniformly generated feature vectors. Thus, even if the offline behavior policy only covers sub-optimal policies, the offline dataset includes rich information about optimal policy. Therefore, in this setting, we have $\mathtt{C}(\pi^*|\rho)\leq O(\mathtt{C}(\pi^{-\varepsilon}|\rho) )$. 
\end{remark}
\fi

{\it Proof sketch.} Our proof is built upon a 2-arm linear contextual bandit instance specified in \citet{he2022reduction}. It has been shown that the regret per episode or the sub-optimality gap is determined by the estimation error of parameter $\theta^*$. Then, we show that the estimation error $\|\hat{\theta} - \theta^*\|_2^2=\mathtt{U}^{\pi^*}(\Dc)^2$ scales in the order of $\Theta(1/\Eb_{\Dc}[(\phi(x,a)^\TT\theta^*_{\bot})^2])$, where $\Dc$ is the available dataset, and $\theta_{\bot}^*$ is orthogonal to $\theta^*$. By choosing $\Dc=\Dc_0$, we have $\mathtt{U}^{\pi^*}(\Dc_0) = 1/\Eb_{\Dc_0}[(\phi(x,a)^\TT\theta^*_{\bot})^2]$. Therefore, $\mathtt{C}(\pi^*|\rho) = N_0/\Eb_{\Dc_0}[(\phi(x,a)^\TT\theta^*_{\bot})^2]$. By choosing $\Dc=\Dc_0\cup\Dc_{t-1}$ at each episode $t$, we conclude that the regret per episode scales in $\Omega(1/\sqrt{N_0\mathtt{C}(\pi^*|\rho) + t}).$

\section{Experimental Results}\label{sec:main exp}

\begin{figure*}[h!]

\begin{minipage}[b]{.19\linewidth}
    \centering
    \begin{tabular}{c|c|c}
        $\rho$ & $\mathtt{C}(\pi^*|\rho)$ & $k$ \\
        \hline
        $\rho_1$  & $1.0$ & $\infty$ \\
        $\rho_2$  & $1.7$ & $5$      \\
        $\rho_3$  & $8.7$ & $-10$
        \end{tabular}
    \subcaption{CE in Bandits}
  \end{minipage}
  \hfill
  \begin{minipage}[b]{.19\linewidth}
    \centering
    \includegraphics[width=0.95\linewidth]{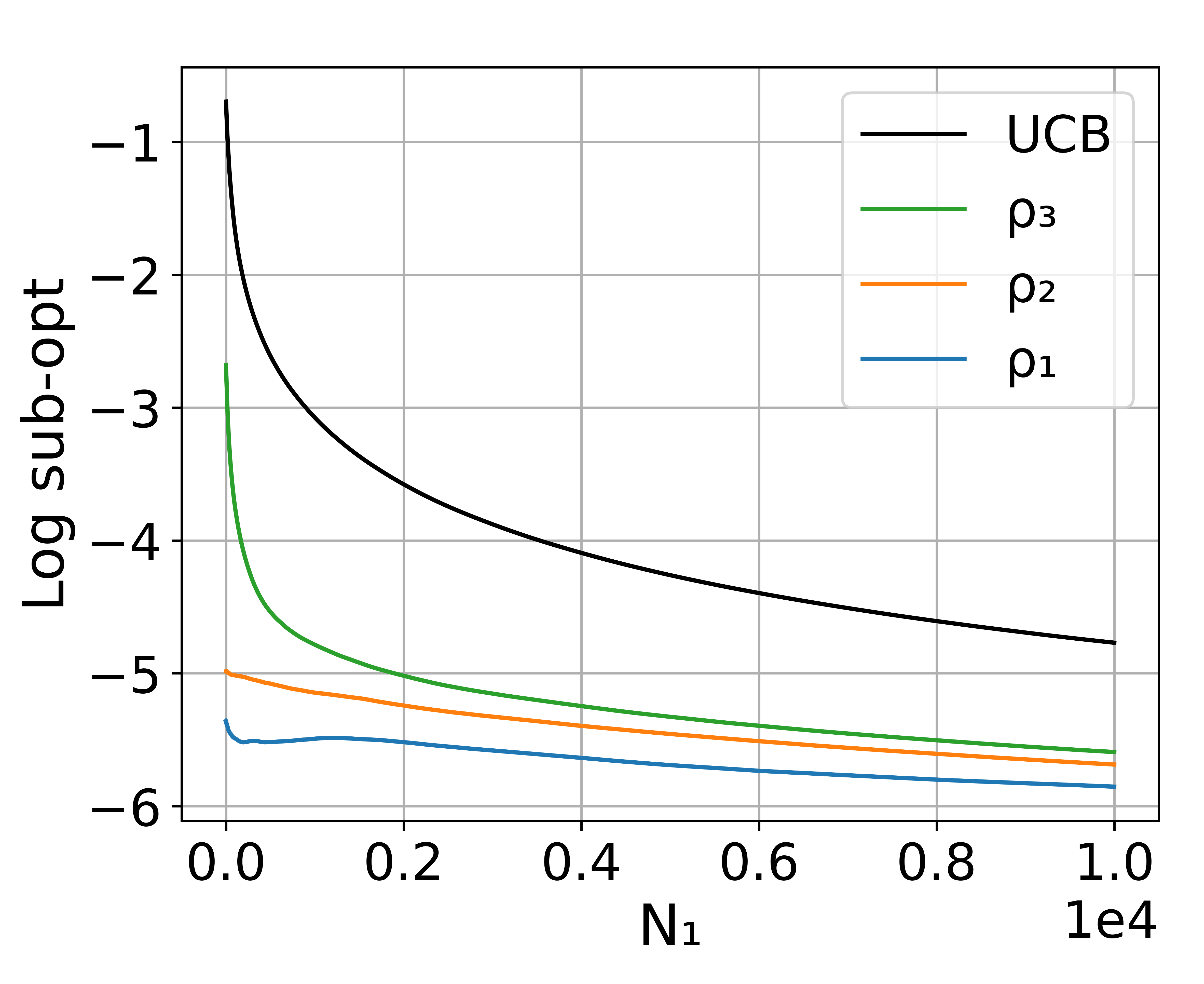}
    \subcaption{SOG v.s. $\rho$} 
    \label{fig:1_1}
  \end{minipage}
  \hfill
  \begin{minipage}[b]{.19\linewidth}
    \centering
    \includegraphics[width=0.95\linewidth]{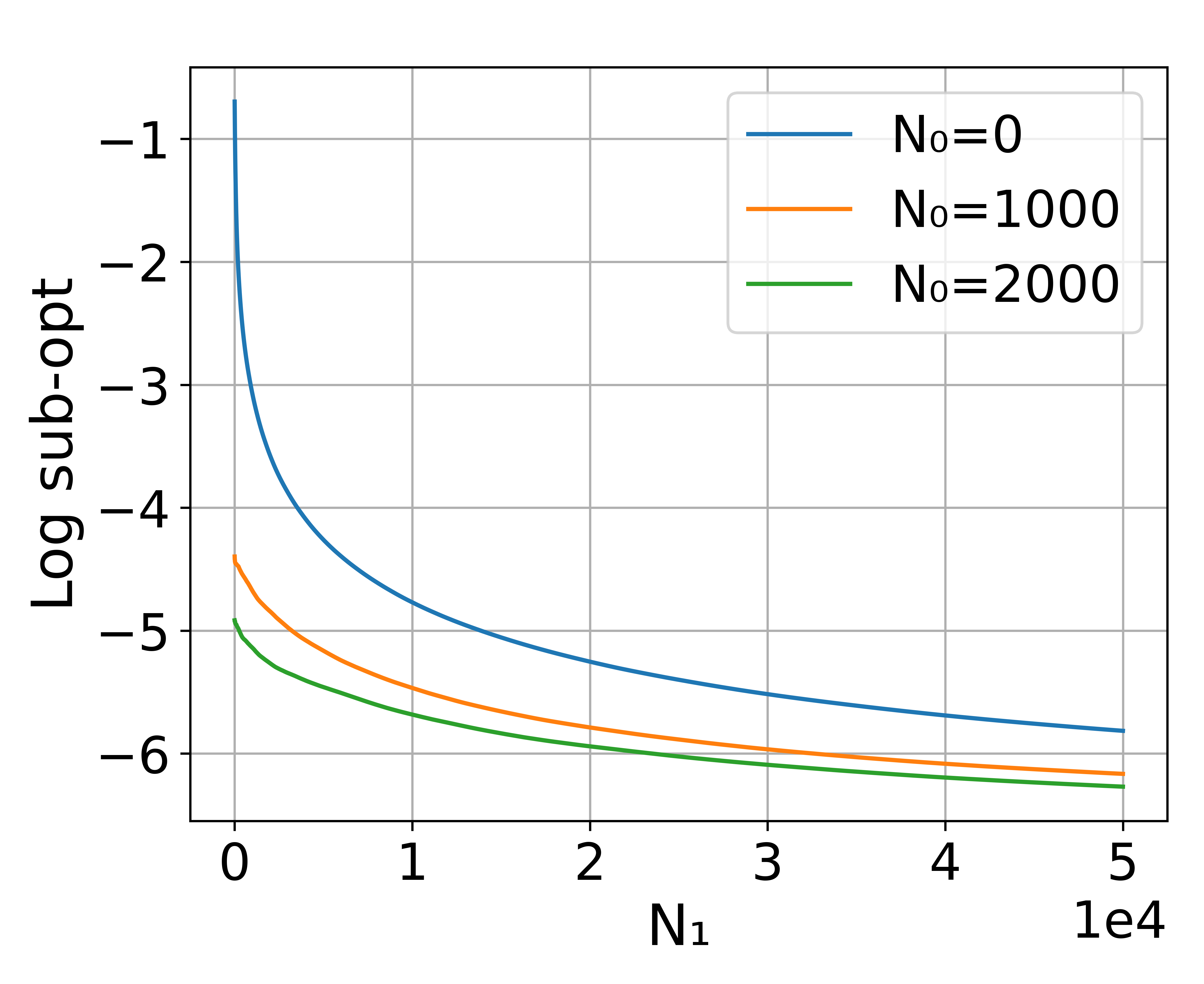} 
    \subcaption{SOG v.s. $N_0$}
    \label{fig:1_2}
  \end{minipage}
  \hfill
  \begin{minipage}[b]{.19\linewidth}
    \centering
    \includegraphics[width=0.95\linewidth]{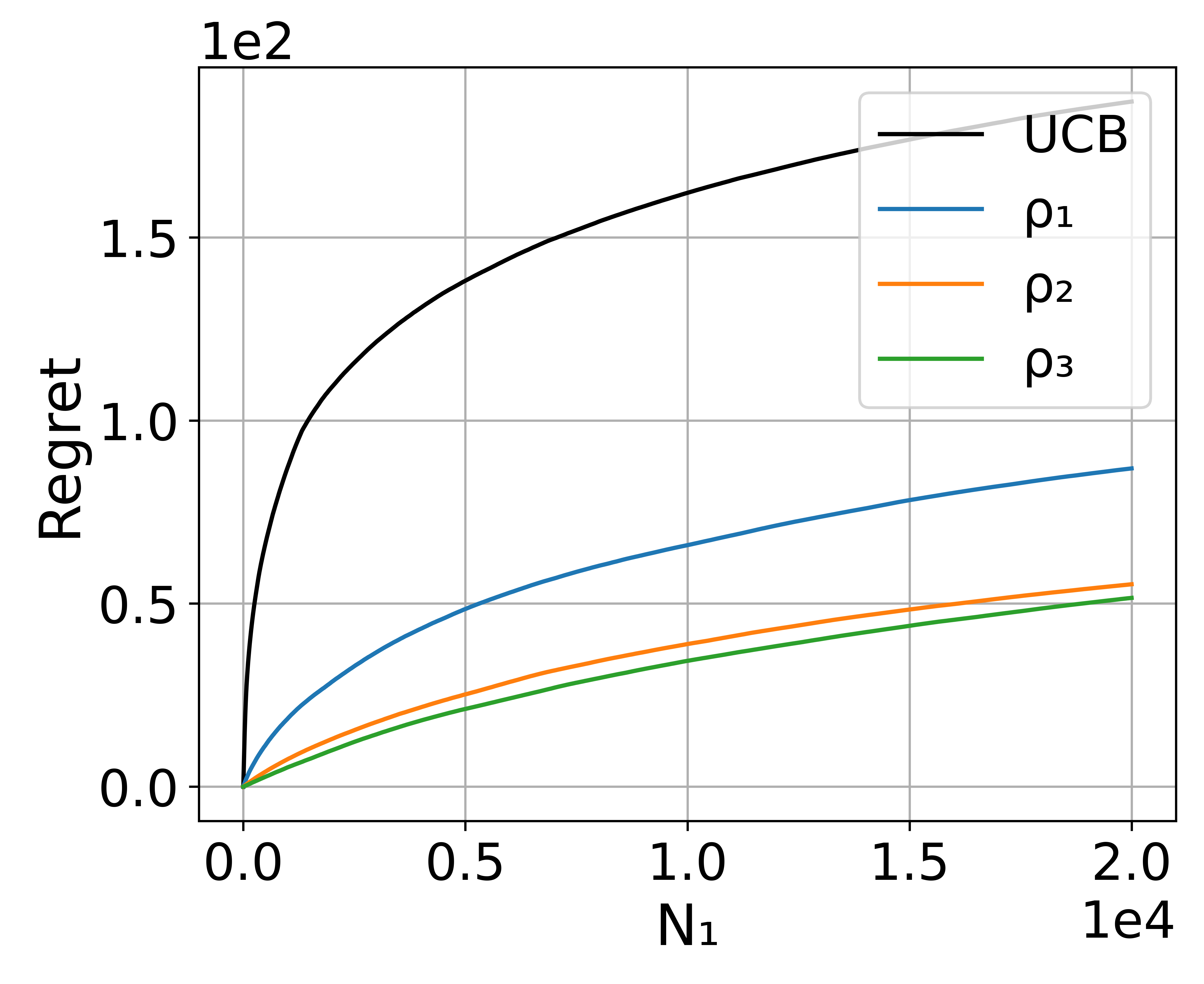} 
    \subcaption{Regret v.s. $\rho$}
    \label{fig:1_3}
  \end{minipage}
   \hfill
  \begin{minipage}[b]{.19\linewidth}
    \centering
    \includegraphics[width=0.95\linewidth]{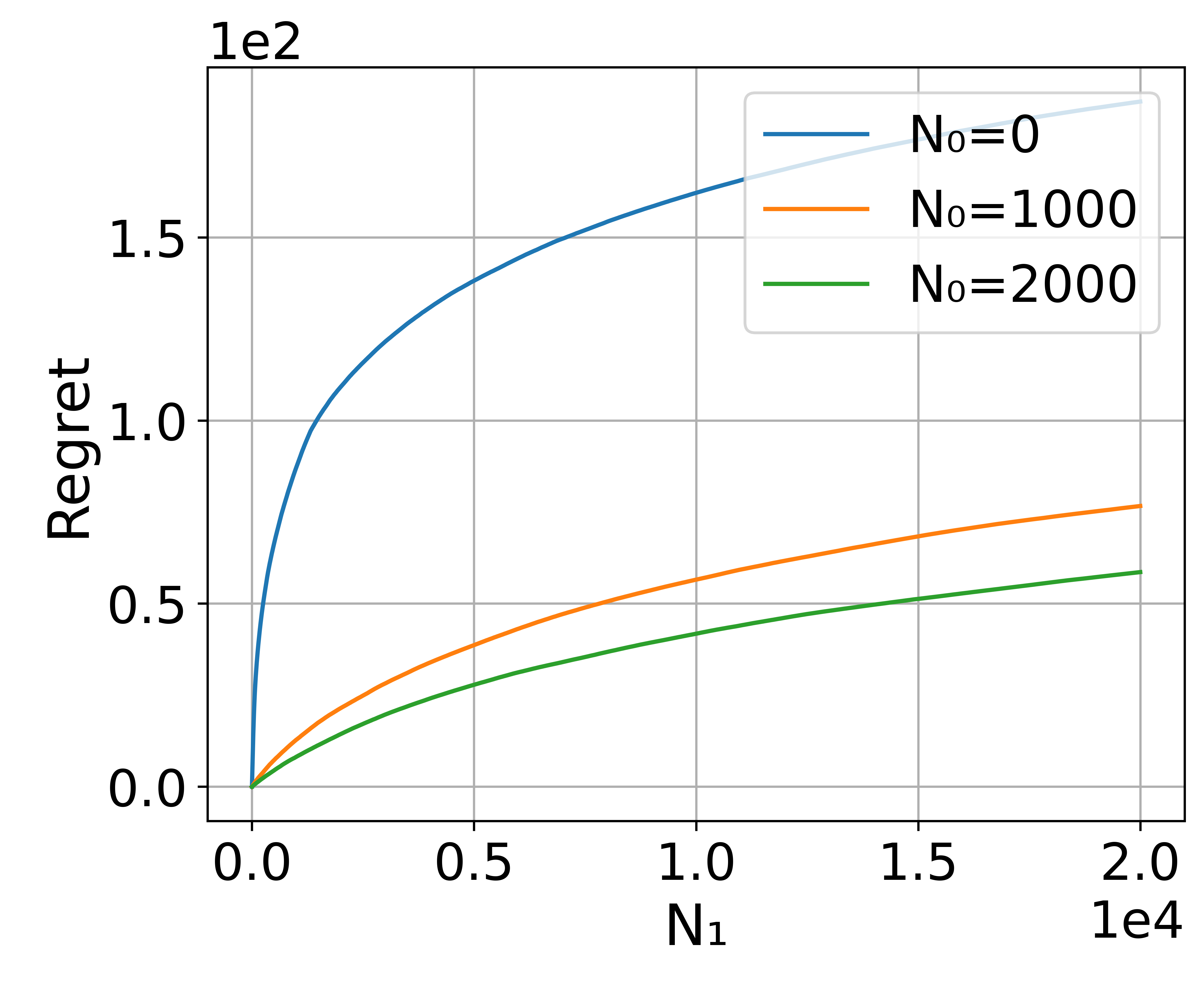} 
    \subcaption{Regret v.s. $N_0$}
    \label{fig:1_4}
  \end{minipage}

  \begin{minipage}[b]{.19\linewidth}
    \centering
    \begin{tabular}{c|c|c}
        $\rho$ & $\mathtt{C}(\pi^*|\rho)$ & $k$\\
        \hline
        $\rho_1$  & $1.0$ & $\infty$ \\
        $\rho_2$  & $2.2$ & $2$ \\
        $\rho_3$  & $4.8$ & $0$
        \end{tabular}
    \subcaption{CE in MDPs}
  \end{minipage}
  \hfill
  \begin{minipage}[b]{.19\linewidth}
    \centering
    \includegraphics[width=0.95\linewidth]{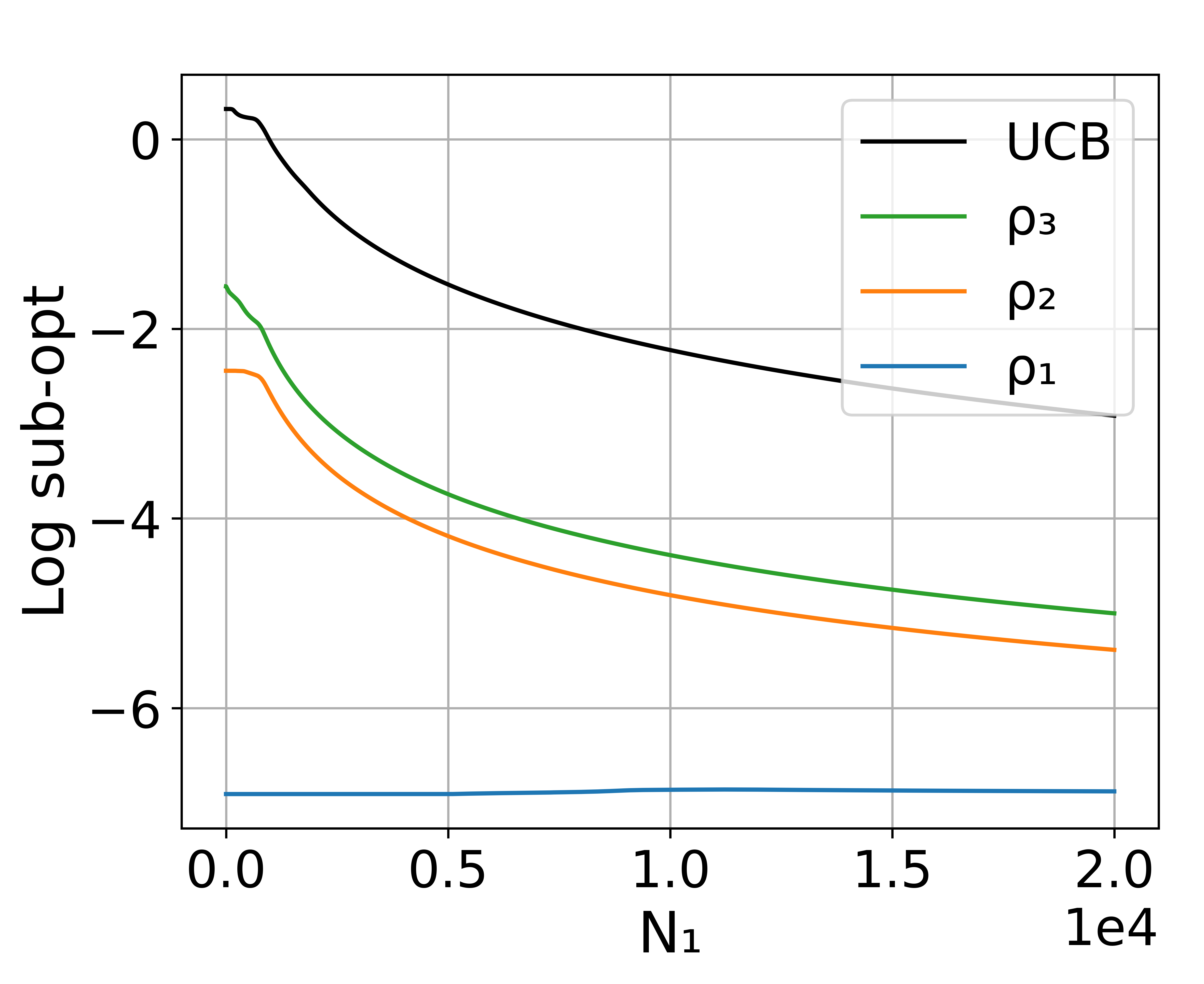} 
    \subcaption{SOG v.s. $\rho$}
    \label{fig:2_1}
  \end{minipage}
  \hfill
  \begin{minipage}[b]{.19\linewidth}
    \centering
    \includegraphics[width=0.95\linewidth]{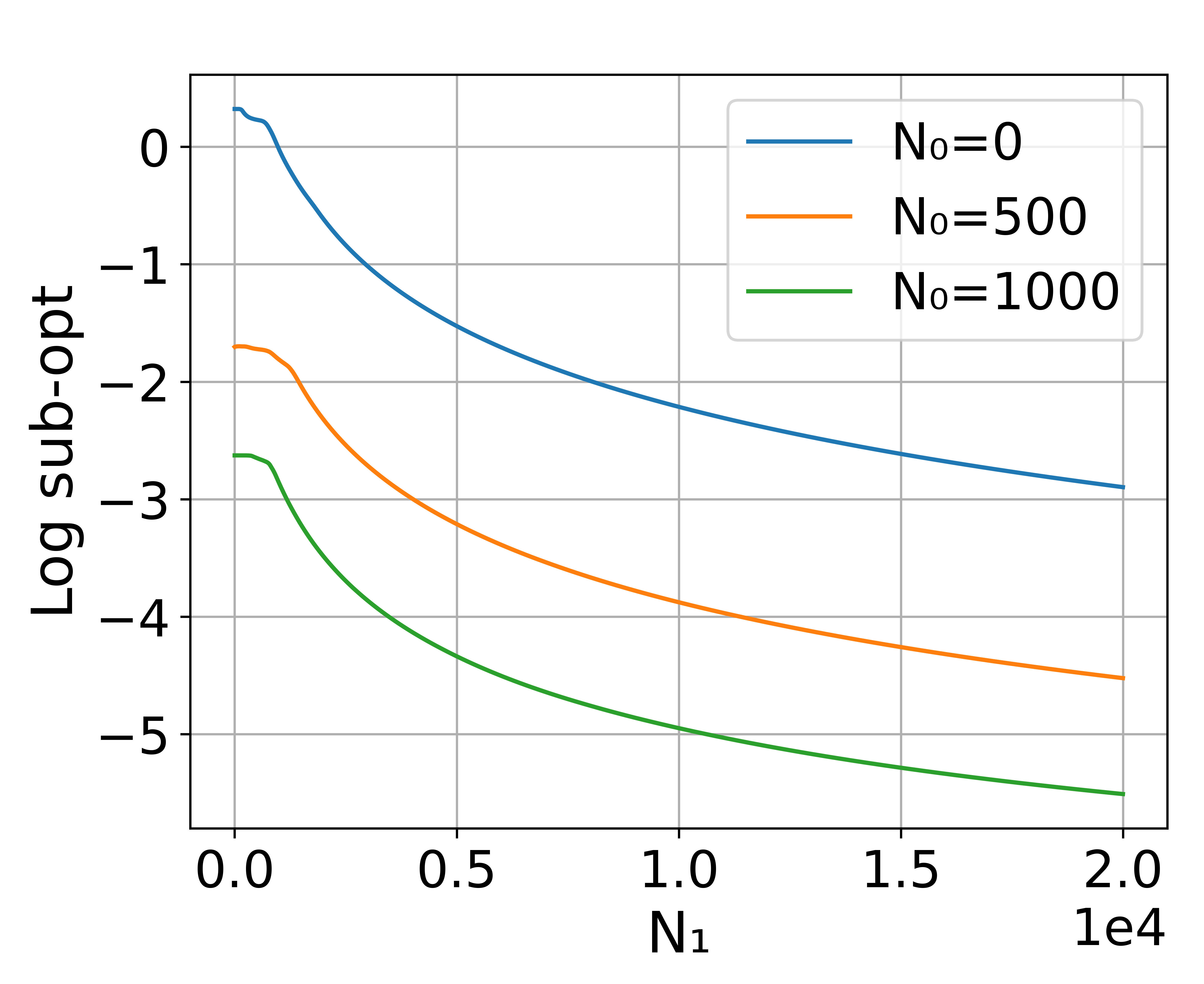} 
    \subcaption{SOG v.s. $N_0$}
    \label{fig:2_2}
  \end{minipage}
  \hfill
  \begin{minipage}[b]{.19\linewidth}
    \centering
    \includegraphics[width=0.95\linewidth]{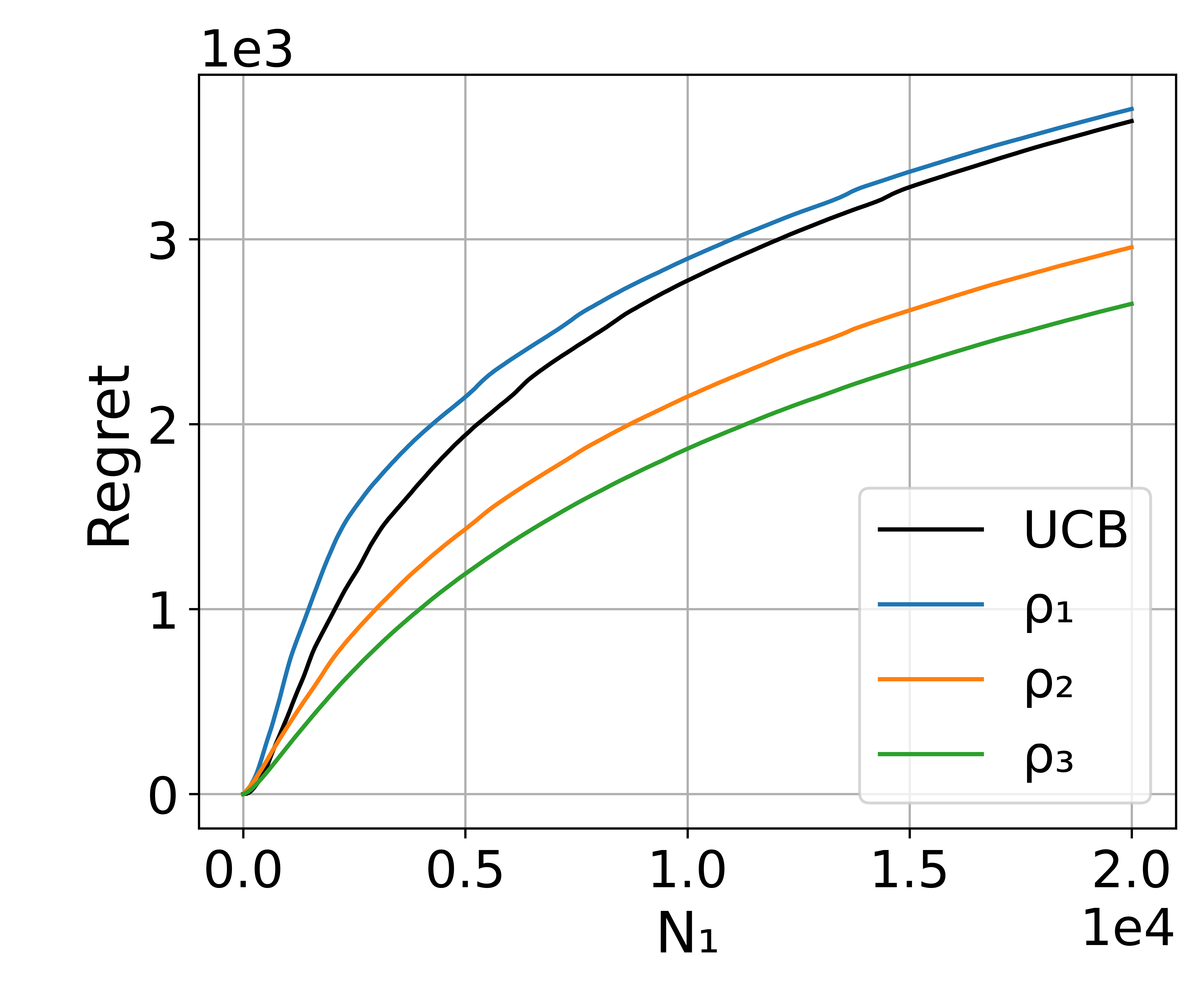} 
    \subcaption{Regret v.s. $\rho$}
    \label{fig:2_3}
  \end{minipage}
   \hfill
  \begin{minipage}[b]{.19\linewidth}
    \centering
    \includegraphics[width=0.95\linewidth]{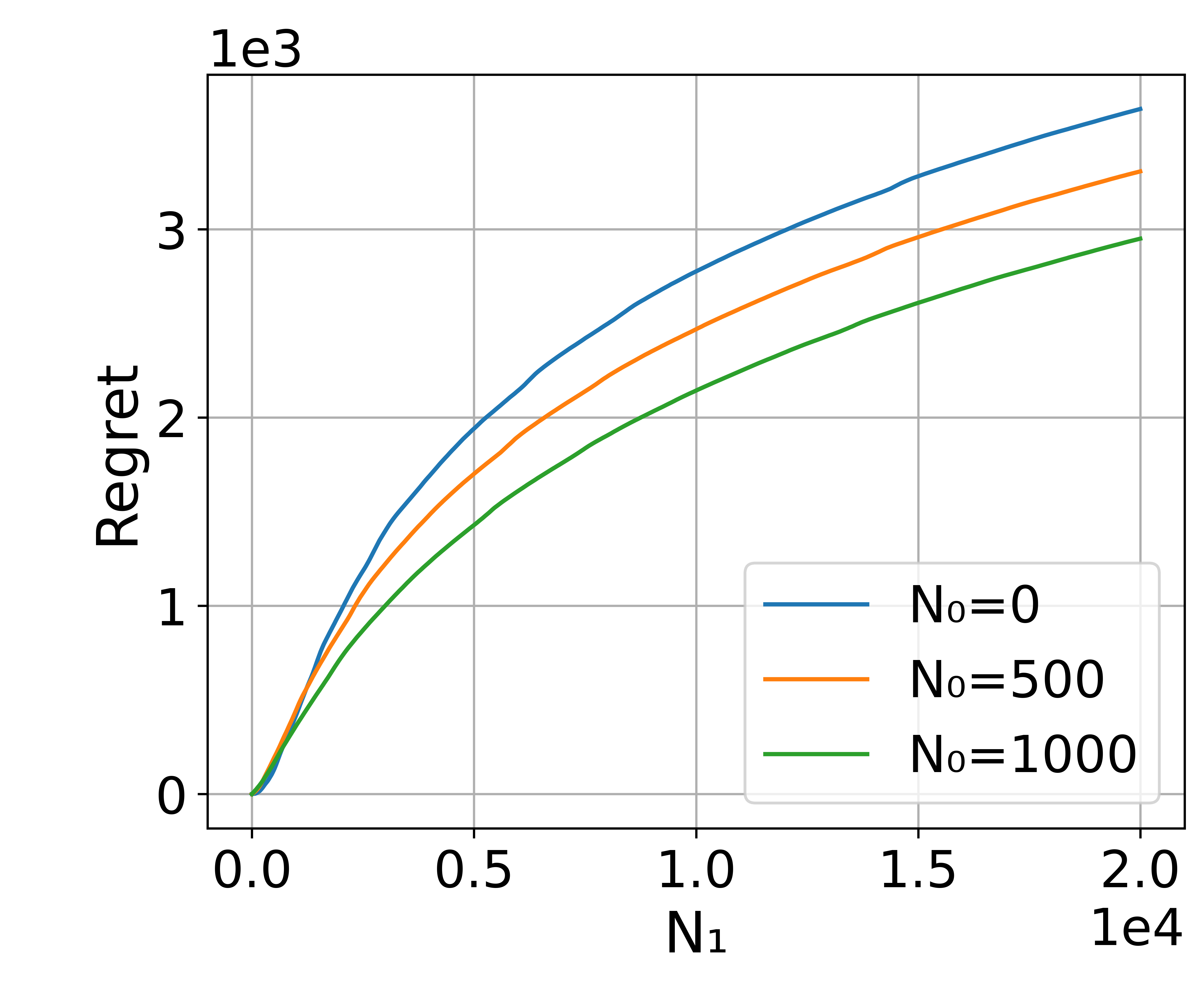} 
    \subcaption{Regret v.s. $N_0$}
    \label{fig:2_4}
  \end{minipage}
  \caption{Experimental results on sub-optimality gap (SOG) and regret for different behavior policies and $N_0$. Figures (a) and (f) show the concentrability coefficients (CE) of three different behavior policies in linear contextual bandits and MDPs, respectively. Figures (b)-(e) are the results on linear contextual bandits. Figures (g)-(j) are results on tabular MDPs. }
  \label{fig:1}
\end{figure*}

In this section, we evaluate the performances of the proposed algorithms in synthetic environments. 
Additional experimental results are provided in \Cref{append:experiment},  including evaluations in a contextual linear bandit constructed from the MovieLens dataset~\citep{harper2015movielens} and a tabular MDP discretized from the Mountain Car environment~\citep{Moore90efficientmemory-based}, implemented in Gymnasium \citep{towers2024gymnasium}.
All of the experiments are conducted on a server equipped with an AMD EPYC 7543 32-core processor and 256GB memory. No GPUs are used. We believe the experiments are also easy to replicate on common PCs.

\textbf{Environment.} We consider two types of environments as described in \Cref{sec:eg}: linear contextual bandits and tabular MDPs.

For the linear contextual bandits, we set $|\Xc|=20$, $|\mathcal{A}|=100$ and $d=10$. At each episode $\{\phi(x,a)\}_{a\in\mathcal{A}}$ is randomly drawn from the unit sphere $\mathcal{S}^{d-1}$. We set $\theta^*$ as a unit vector with the first element being 1. 
In addition, the reward is given by $r_t = \phi(x_t,a_t)^\TT\theta^* + \xi_t$, where $\xi_t\in[-1,1]$ is independently sampled from the uniform distribution.

For tabular MDPs, we set $H=3$, $|\Xc|=5$ and $|\mathcal{A}|=10$. The initial states are uniformly and randomly chosen at each episode.
The transition probability $P_h(\cdot|s,a)$ at each step and state-action pair is uniformly sampled from the probability simplex.
The reward $r_h(s,a)$ is uniformly generated from $[0,1]$ and is assumed to be known to the agent for simplicity.

\textbf{Offline Dataset Collection.} 
We adopt the Boltzmann policy \citep{szepesvari2022algorithms} as the behavior policy. Under the Boltzmann policy, actions are taken randomly according to $\rho_h(a|x) = \frac{\exp\{k Q_h(x,a)\}}{\sum_{a\in \mathcal{A}} \exp\{k Q_h(x,a)\}}$, where $k\in\Rb$, and $Q_h(x,a)$ is the optimal Q-value function starting from $(x,a)$ at time step $h$. Note that a larger $k$ makes $\rho$ closer to the optimal policy, and therefore makes $\mathtt{C}(\pi^*|\rho)$ smaller. In particular, $Q_1(x,a)=r(x,a)$ in linear contextual bandits.


We consider three behavior policies, denoted as $\rho_1$, $\rho_2$, and $\rho_3$, by setting different $k$ of the Boltzmann policy. In \Cref{fig:1}(a) and (f), we list the values of $k$ used to generate the Boltzmann policy and the concentrability coefficient $\mathtt{C}(\pi^{-\epsilon}|\rho)$ for the two environments. 
When $k=\infty$, $\rho_1$ is the optimal policy, so it has the best coverage of the optimal policy, and $\mathtt{C}(\pi^*|\rho)=1$. As $k$ decreases, the policy becomes further away from the optimal policy, thus $\mathtt{C}(\pi^*|\rho)$ increases. In addition, in both environments, we ensure that  $\rho_2$ and $\rho_3$ are sub-optimal polices and thus, $\mathtt{C}(\pi^{-\varepsilon}|\rho_1) > \mathtt{C}(\pi^{-\varepsilon}|\rho_2) > \mathtt{C}(\pi^{-\varepsilon}|\rho_3)$.

Finally, in \Cref{fig:1_1} and \Cref{fig:2_1}, we fix the offline dataset size $N_0$ as $2000$ and $1000$, respectively. In \Cref{fig:1_2} and \Cref{fig:2_2}, we fix the behavior policy as $\rho_2$.


\textbf{Results.} We present the experiment results in \Cref{fig:1}. For both environments, we evaluate the sub-optimality gap and the regret with different offline behavior policies and varying numbers of offline trajectories. For each environment, we conduct $100$ trials and plot the sample average sub-optimality gap or regret as a function of online time steps $N_1$. 
The baseline is the pure UCB algorithm without any offline dataset.

\textit{Sub-optimality gap.} Our experimental results confirm our theoretical findings in \Cref{thm:suboptimality}. Specifically, augmentation using offline data exhibits superior performance than pure online results in both environments. More importantly, the smaller the concentrability coefficient $\mathtt{C}(\pi^*|\rho)$ (\Cref{fig:1_1,fig:2_1}), or the larger the offline dataset size (\Cref{fig:1_2,fig:2_2}), the smaller the sub-optimality gap under the same amount online episodes.


\textit{Regret.}
We then compare the results of regret minimization in both the contextual linear bandit and the tabular MDP. Recall that while $\mathtt{C}(\pi^*|\rho_1) < \mathtt{C}(\pi^*|\rho_2) < \mathtt{C}(\pi^*|\rho_3)$, we have $\mathtt{C}(\pi^{-\varepsilon}|\rho_1) > \mathtt{C}(\pi^{-\varepsilon}|\rho_2) > \mathtt{C}(\pi^{-\varepsilon}|\rho_3)$. As a result, \Cref{fig:1_3,fig:2_3} shows that the regret decreases as the $\mathtt{C}(\pi^{-\varepsilon}|\rho)$ decreases and \Cref{fig:1_4,fig:2_4} shows that larger offline dataset size leads to smaller regret. These experimental results align with our theoretical findings in \Cref{thm:main regret}. Finally, we remark that in \Cref{fig:2_3}, the pure online algorithm is slightly better than the hybrid algorithm with offline policy being the optimal policy. This outcome arises because the hybrid algorithm will prioritize exploring actions that are less explored in the offline dataset. When the offline dataset primarily consists of optimal actions, the hybrid algorithm takes sub-optimal actions more frequently than pure online algorithms, leading to slightly higher regret. Nevertheless, the performance is still comparable with the baseline and aligns with \Cref{thm:main regret}.

\textit{Key insight.} The contrasting performances of the same behavior policy under sub-optimality gap minimization and regret minimization problems highlight the need for different offline datasets for these two tasks. Specifically, if the objective is to find a near-optimal policy and the cost of online exploration is negligible, then an offline dataset that focuses on covering the optimal policy is sufficient. However, if the goal is to minimize the regret, it is more effective to collect offline data using various sub-optimal policies rather than the optimal policy.

\if{0}
and compare the performances of the algorithms in \Cref{fig:1_3} and \Cref{fig:2_3}, respectively.  We note for our algorithm with offline datasets collected under the three behavior policies, the regret increases in the order of $\rho_3$, $\rho_2$ and then $\rho_1$, where $\rho_1$ corresponds to the optimal policy. This indicates that datasets collected under good policies (in terms of closeness to the optimal policy) may not be the most beneficial for regret minimization. This finding aligns with our theoretical results in \Cref{thm:main regret}, which states that achieving small regret
requires a small maximum concentrability coefficient of sub-optimal policies (i.e., $\mathtt{C}(\pi^{-\varepsilon}|\rho)$) is small. 
Computing $\mathtt{C}(\pi^{-\epsilon}|\rho)$ is challenging due to the difficulty in identifying the most challenging suboptimal policy. However, the coefficient tends to be large when $\mathtt{C}(\pi^*|\rho)$ is small. This relationship arises because the better a policy covers the optimal policy, the more challenging it becomes for that policy to encompass all suboptimal policies. Thus, we qualitatively use the opposite ordering of $\mathtt{C}(\pi^*|\rho)$ to characterize the ordering of $\mathtt{C}(\pi^{-\epsilon}|\rho)$.
Among those three behavior policies, the ascending order of $\mathtt{C}(\pi^*|\rho)$ is $\rho_1,\rho_2,\rho_3$, which approximates the descending ordering of $\mathtt{C}(\pi^{-\epsilon}|\rho)$.
As shown in \Cref{fig:1_3} and \Cref{fig:2_3}, the order of regret, from highest to lowest, is $\rho_1,\rho_2,\rho_3$, which is consistent with \Cref{thm:main regret}.
Besides, the regret under the pure UCB algorithm is generally higher than our hybrid algorithm with the offline datasets, indicating the benefit of including offline data for regret minimization. 
When the behavior policy is $\rho_1$ in the MDP environment, the regret of hybrid learning is slightly higher than that of the pure online UCB algorithm. 
This is reasonable because $\rho_1$ is the optimal policy. When the optimal policy collects the offline dataset, the online UCB algorithm will tend to explore where the optimal policy does not cover well, leading to a marginally higher regret.
Additionally, \Cref{fig:1_4} and \Cref{fig:2_4} compare the regrets with varying numbers of offline dataset size $N_0$.
It clearly demonstrates that increasing $N_0$ reduces the online learning regret, corroborating the theoretical results in \Cref{thm:main regret}. 
\fi

\section{Conclusion} 
In our paper, we developed a general hybrid RL framework to minimize the sub-optimality gap and the online learning regret. 
The framework achieves performance bounds of $\tilde{O}(1/\sqrt{N_0/\mathtt{C}(\pi^*|\rho) + N_1} )$ for the sub-optimality gap and $\tilde{O}(\sqrt{N_1}\sqrt{N_1/(N_0/\mathtt{C}(\pi^{-\varepsilon}|\rho) + N_1)})$ for the regret, where $\mathtt{C}(\pi^*|\rho)$ and $\mathtt{C}(\pi^{-\varepsilon}|\rho)$ are two concentrability coefficients for optimal policy and sub-optimal policies, respectively. Our results demonstrate the benefits of integrating offline data with online interactions. More importantly, the same behavior policy $\rho$ leads to different performances in sub-optimality gap and regret minimization. Our experimental results corroborated our theoretical findings. In addition, we particularized our framework to two specific settings: the contextual linear bandit setting and the tabular MDP setting.
We also derived lower bounds for the hybrid RL problem, showing that our approach is nearly optimal.
Our results highlight the advantages of leveraging offline datasets for more efficient online learning and provide insights into the selection of offline datasets and policies for different online tasks. 


\begin{acknowledgements}
     The work of R. Huang, D. Li and J. Yang was supported in part by the U.S. National Science Foundation (NSF) under the grants CNS-1956276 and CNS-2114542. The work of C. Shen was supported in part by NSF awards 2029978, 2002902, and 2143559
\end{acknowledgements}

\bibliography{main}

\newpage
\appendix
\onecolumn

\title{Augmenting Online RL with Offline Data is All You Need:\\ A Unified Hybrid RL Algorithm Design and Analysis \\(Supplementary Material)}
\maketitle


\section{Related Works}
\label{sec:related_works}

\textbf{Offline Reinforcement Learning (RL).} It has been known that the performance of offline RL depends critically on the concept of ``coverage''. Early works \citep{munos2008finite,ross2012agnostic,chen2019information,duan2020minimax} largely assume full coverage, which suggests that the data generated by offline policy can cover the data distribution generated by any policy. This restriction has been alleviated by recent results \citep{kumar2020conservative,jin2021pessimism,rashidinejad2021bridging,xie2021bellman,zanette2021provable,uehara2022pessimistic}. These results provide a better understanding of the pessimism principle under partial coverage assumption, which only requires the offline data to cover the data distribution of the optimal policy or a comparator policy. 

\textbf{Hybrid RL.} 
Several recent works have studied the hybrid policy learning problem for reinforcement learning. 
\citet{xie2021policy} focus on the sample complexity improvement in episodic tabular MDPs when both offline and online datasets are used. 
They show the single-policy concentrability coefficient $C^*$ between offline behavior policy and optimal policy is crucial, and the sample complexity required to achieve an $\epsilon$-optimal policy is $\tilde{O}(H^3 S \min{(A, C^*)}/\epsilon^2)$, where $H, S,A$ are the episodic length, number of states, and number of actions of the MDPs, respectively. {While it achieves the best sample complexity in both offline and online RL,} it requires the knowledge of the concentrability and the access of the behavior policy. 
\citet{li2023rewardagnostic} also consider tabular MDPs with a fixed offline dataset and do not assume the knowledge of $C^*$.
They introduce a new single-policy partial concentrability coefficient $C^*(\sigma)$ which generalizes the original $C^*$ by allowing the behavior policy only cover a proportion $\sigma$ of the state-action pairs. 
With the new concentrability coefficient and an imitation approach, the show that an $\epsilon$-optimal policy can be obtained if $N_0+N_1\ge \tilde{O}(H^3SC^*(\sigma)/\epsilon^2)$ and $N_1 \ge \tilde{O}(H^3S\min(H\sigma, 1)/\epsilon^2)$ for a $\sigma\in[0,1]$, where $N_0$ and $N_1$ are the numbers of offline and online samples, respectively.
By choosing proper $\sigma\in[0,1]$, such sample complexity outperforms pure online and offline RL algorithms.

Beyond the tabular setting, \citet{wagenmaker2023leveraging} consider linear MDPs and also design a new online-to-offline concentrability coefficient $C_\text{o2o}$. 
Compared with the single-policy concentrability coefficient $C^*$, the coefficient $C_\text{o2o}$ considers not only the coverage from the offline dataset but also the coverage of the potential online dataset.
Under the assumptions that the offline dataset has good coverage and the number of online samples is not large, they show the leading term of the sample complexity can be reduced from $\tilde{O}(1/\epsilon^2)$ to $\tilde{O}(1/\epsilon^{8/5})$.

There are also works on hybrid RL with general function approximation.
\citet{song2022hybrid} propose hybrid Q-learning, which achieves regret in regret $\tilde{O}(\max\{C, 1\} \sqrt{d N_1})$ and shows empirical advantages in Atari environments. However, when the offline dataset does not have good coverage, the coverage coefficient $C$ will be much greater than $1$ and the regret could be worse than pure online learning, which is also mentioned in \citet{wagenmaker2023leveraging}.
\citet{tan2024natural} study a Global Optimism based on Local Fitting (GOLF) ~\citet{jin2021bellman}-based algorithm for hybrid RL with general Q-function approximation, and show that it can achieve a regret of $\tilde{O}(\sqrt{d N_1^2/N_2}+\sqrt{d N_1})$ in stochastic linear bandits. We also note a concurrent work~\citet{tan2024hybrid} studied the sub-optimality gap and regret simultaneously in linear MDPs. However, their results relies on all-policy concentrability coefficient and is less general than our results.

\textbf{Online Bandits.} 
The study of multi-armed bandit problems traces back to the original work by \citet{thompson1933likelihood} for adaptive clinical trials. 
Many classical algorithms have been proposed, including Thompson Sampling \citep{chapelle2011empirical,agrawal2012analysis}, 
and the family of Upper Confidence Bound (UCB) algorithms \citep{Lai:1985,Auer:2002,audibert2009exploration,abbasi2011improved,garivier2011kl,cappe2013kullback}. 
These algorithms balance the intrinsic exploration-exploitation tradeoff and achieve the optimal learning regret. 
The linear bandit model, as a generalization of finite armed bandits with linear reward structure, has also been well studied. 
\citet{Auer:2002} extend the UCB algorithm to stochastic linear bandits problem and achieve regret $\tilde{O}(\sqrt{N_1})$ over time horizon $N_1$. 
\citet{dani2008stochastic,abbasi2011improved} further match the lower bound $\Omega(d\sqrt{N_1})$ of \citet{dani2008stochastic} up to logarithmic factors, where $d$ is the feature dimension.
The celebrated LinUCB has been proposed and analyzed in \citet{li2010contextual}, which consider linear contextual bandits and achieve regret $\tilde{O}(\sqrt{K d N_1})$ for $K$-armed disjoint linear model.
Other than UCB-type algorithms, \citet{agrawal2013thompson} use Thompson Sampling in linear contextual bandit and achieves regret $\tilde{O}(d^2\sqrt{N_1})$.
Additionally, \citet{soare2014best, jedra2020optimal} use optimal design for best-arm identification with $\tilde{O}(d/\epsilon^2)$ samples. \citet{yang2022minimax} and \citet{wagenmaker2021experimental} utilize optimal design in linear bandits to achieve minimax optimal results for fixed-budget best-arm identification and regret minimization, respectively.

\textbf{Offline Bandits.} 
Research on offline bandits has been limited. \citet{rashidinejad2021bridging} develop a pessimism-based algorithm and match their information-theoretic lower bound on the sub-optimality gap $\Omega\left(\sqrt{S C^*/N_0}\right)$ for finite-arm finite-context bandits, where $C^*$ is the concentrability coefficient $C^*$ representing the coverage of the offline dataset on the optimal policy, $S$ is the number of the contexts and $N_0$ is the size of the offline dataset. 
\citet{li2022pessimism} propose a family of pessimistic learning rules for offline linear contextual bandits and pprove that the suboptimality gap scales in $\tilde{O}(\sqrt{d C^\star/N_0})$ for fixed contexts.

\textbf{Hybrid Bandits.}
Hybrid policy learning in bandits is related to the \emph{warm-start} bandits problem. 
Several works \citep{li2010contextual, sharma2020warm, silva2023user, zhang2019warm} investigate utilizing the offline data to improve the online performance under different settings. 
\citet{shivaswamy2012multi} study stochastic bandit with finite arms, where 
the offline dataset is utilized to approximate the confidence bound in UCB algorithms. It shows that the number of pulling a sub-optimal arm $a$ with reward gap $\Delta_a$ scales as $\tilde{O}(\max(0, 8/\Delta^2_a - N_{0,a}))$, where $N_{0,a}$ is the number of times pulling arm $a$ in the offline dataset. 
\citet{oetomo2023cutting} augment the existing Thompson Sampling algorithms by initialing the covariance matrix and reward vector in linear contextual bandits using the offline data. They achieve regret $\tilde{O}(\sqrt{N_1\log((\det(V_{N_0+N_1}))/\det(V_{N_0}))})$, where $V_{N_0+N_1}$ and $V_{N_0}$ are the covariance matrices constructed with both online and offline datasets and offline dataset only, respectively. Since the improvement is logarithmic, the regret advantage is marginal.
\citet{agrawal2023optimal} study the lower bound of the $\delta$-correct best-arm identification in stochastic $K$-armed bandits given the offline dataset generated by an unknown policy, and further design an algorithm whose instant-dependent sample complexity matches the lower bound.
Beyond these, \citet{cheung2024leveraging} consider generalized hybrid learning in tabular stochastic bandits where the offline dataset can have different reward distributions than the online environment. 
When both distributions match, their result reduces to that of \citet{shivaswamy2012multi}, but the work demonstrates the transferability of hybrid learning.
Our paper focuses on the hybrid stochastic linear bandit problem, and we develop algorithms that can simultaneously have better sub-optimality and regret than online or offline learning, which has not been studied before.

\if{0}

\begin{algorithm}
\caption{UCB-VI}
\label{alg:ucbvi}
    \begin{algorithmic}
        \STATE {\bfseries Initialize:} $\hat{P}_{t-1}$, $\hat{r}_{t-1}$, $N_h^{(t-1)}:\Xc\times \mathcal{A}\rightarrow \mathbb{N}$.
        \FOR{$t=1,\ldots, N_1$}
        \STATE Choose policy 
        \ENDFOR
    \end{algorithmic}
\end{algorithm}

\begin{algorithm}
\caption{Lin-UCB}
\label{alg:linucb}
    \begin{algorithmic}
        \STATE {\bfseries Initialize} $\hat{\Lambda}_0$, $b_0$ $\beta$
        \FOR{$t=1,\ldots, N_1$}
        \STATE Compute $\hat{\theta}_{t-1} = \hat{\Lambda}_{t-1}^{-1}b_{t-1}$
        \STATE Observe context $x_t$, select $a_t = \arg\max_{a\in\mathcal{A}} \phi(x_t,a)^\TT\hat{\theta}_{t-1} + \beta\|\phi(x_t,a)\|_{\hat{\Lambda}_{t-1}^{-1}}$, receive reward $r_t$.
        \STATE Update $\hat{\Lambda}_t = \hat{\Lambda}_{t-1} + \phi(x_t,a_t)\phi(x_t,a_t)^\TT$, $b_{t} = b_t + \phi(x_t,a_t)r_t$.
        \ENDFOR
    \end{algorithmic}
\end{algorithm}

\begin{algorithm}
\caption{BPI-UCB}
\label{alg:bpiucb}
    \begin{algorithmic}
        \STATE {\bfseries Initialize:} $\hat{\Lambda}_0$, $b_0$ $\beta$
        \FOR{$t=1,\ldots, N_1$}
        \STATE Observe context $x_t$, select $a_t = \arg\max_{a\in\mathcal{A}}\|\phi(x_t,a)\|_{\hat{\Lambda}_{t-1}^{-1}}$, receive reward $r_t$.
        \STATE Update $\hat{\Lambda}_t = \hat{\Lambda}_{t-1} + \phi(x_t,a_t)\phi(x_t,a_t)^\TT$, $b_{t} = b_t + \phi(x_t,a_t)r_t$.
        \ENDFOR
        \STATE Compute $\hat{\theta} = \hat{\Lambda}_{N_1}^{-1}b_{N_1}$
        \STATE {\bfseries Output:} $\pi(x_t) = \arg\max_{a\in\mathcal{A}} \phi(x_t,a)^\TT\hat{\theta}_{t-1} - \beta\|\phi(x_t,a)\|_{\hat{\Lambda}_{t-1}^{-1}} $
    \end{algorithmic}
\end{algorithm}
\fi

\section{Missing Proofs of Main Results}
In this work, our analysis is built upon high-probability bound. That is, all inequalities and equations hold with probability at least $1-\delta$, where $\delta>0$ can be arbitrarily small with a $O(\log(1/\delta))$-factor blow-up in inequalities. 

In particular, by taking $\hat{V}^{\rho} = \frac{1}{|\Dc_0|}\sum_{\tau\in\Dc_0} \sum_{h=1}^H r_h(\tau)$, we obtain a near-optimal upper bound for $\mathtt{U}(\rho|\Dc_0)$ through Azuma-Hoeffding inequality. Mathematically, the following inequality holds with probability at least $1-\delta$.
\[ \mathtt{U}(\rho|\Dc_0) \leq \sqrt{ \frac{2\log(1/\delta)}{N_0} }  =  \tilde{O}(1/\sqrt{N_0}). \]

\begin{theorem}[Restatement of \Cref{thm:suboptimality}]\label{thm:gap appendix}
    Let $\mathtt{Alg}$ be a confidence based algorithm and satisfy the conditions in \Cref{eluder condition}, $\hat{\pi}$ be the output policy of \Cref{alg:hybrid}. Suppose $\pi^*$ is an optimal policy. Then, the sub-optimality gap $\hat{\pi}$ is 
    \[
        \text{Sub-opt}(\hat{\pi}) = \tilde{O}\left(   \frac{ C_{\mathtt{Alg}} }{ \sqrt{N_0/\mathtt{C}(\pi^*|\rho)  + N_1}  } \right),
    \]
    where $N_0$ is the number of offline samples, $N_1$ is the number of online samples, $\mathtt{C}(\pi^*|\rho)$ is the concentrability coefficient, and $C_{\mathtt{Alg}}$ is defined in \Cref{eluder condition}.

\end{theorem}

\begin{proof}

    We have
    \begin{align*}
        \text{Sub-opt}(\hat{\pi}) & = V_{\Mc^*}^{\pi^*} - V_{\Mc^*}^{\hat{\pi}} \\
        & \overset{(a)}\leq \hat{V}^{\pi^*} - \hat{V}^{\hat{\pi}} + \hat{\mathtt{U}}_{\mathtt{Alg}} (\pi^*|\Dc_0\cup\Dc_{N_1}) + \hat{\mathtt{U}}_{\mathtt{Alg}} (\hat{\pi}|\Dc_0\cup\Dc_{N_1}) \\
        &\overset{(b)}\leq 2 \hat{\mathtt{U}}_{\mathtt{Alg}} (\pi^*| \Dc_0\cup\Dc_{N_1} )
    \end{align*}
where $(a)$ follows from the definition of confidence based algorithm, and $(b)$ is due to the optimality of $\hat{\pi}$.

By \Cref{def:Coverbility}, we have
\begin{align*}
    \hat{\mathtt{U}}_{\mathtt{Alg}} (\pi^*| \Dc_0\cup\Dc_{N_1} )& \leq \hat{\mathtt{U}}_{\mathtt{Alg}} (\pi^*| \Dc_0) \\
    &\leq C_{\mathtt{Alg}}\mathtt{U}_{\Mc^*}(\pi^*)\\
    &\leq C_{\mathtt{Alg}} \sqrt{\mathtt{C}(\pi^*|\rho)} \mathtt{U}(\rho|\Dc_0) \leq \tilde{\Theta}\left( C_{\mathtt{Alg}} \sqrt{ \frac{\mathtt{C}(\pi^*|\rho)}{N_0} }\right),
\end{align*}
where the first inequality is due to the fact that $\hat{\mathtt{U}}_{\mathtt{Alg}}(\pi|\Dc)\leq \hat{\mathtt{U}}_{\mathtt{ALg}} (\pi|\Dc')$ if $\Dc'\subset\Dc$.

On the other hand, we have
\begin{align*}
    \hat{\mathtt{U}}_{\mathtt{Alg}}(\pi^*|\Dc_0\cup\Dc_{N_1}) & =  \frac{1}{N_1} \sum_{t=1}^{N_1} \hat{\mathtt{U}}_{\mathtt{Alg}} (\pi^*|\Dc_0\cup\Dc_{N_1} ) \\
    &\overset{(a)}\leq \frac{1}{N_1} \sum_{t=1}^{N_1} \hat{\mathtt{U}}_{\mathtt{alg}}(\pi^*|\Dc_{t-1})\\
    &\overset{(b)}\leq \frac{1}{N_1} \sum_{t=1}^{N_1} \hat{\mathtt{U}}_{\mathtt{Alg}}(\pi_t|\Dc_{t-1})\\
    &\overset{(c)}\leq  \tilde{O}\left(\frac{1}{N_1} \sqrt{N_1 C_{\mathtt{Alg}}^2 } \right) \\
    & = \tilde{O}\left(C_{\mathtt{Alg}}\sqrt{\frac{1}{N_1}}\right),
\end{align*}
where $(a)$ is due to the fact that the uncertainty level decreases as the available dataset increases, $(b)$ follows from the optimality of $\pi_t$, and $(c)$ is due to the combination of the Cauchy's inequality and eluder-type condition~(\Cref{eluder condition}).

Therefore,
\begin{align*}
    \text{Sub-opt}(\hat{\pi}) &\leq \tilde{O}\left( C_{\mathtt{Alg}} \min\left\{\sqrt{\frac{ \mathtt{C}(\pi^*|\rho) }{ N_0} }, \sqrt{\frac{1}{N_1}} \right\} \right) \\
    & = \tilde{O}\left(C_{\mathtt{Alg}} \sqrt{ \min\left\{ \frac{ \mathtt{C}(\pi^*|\rho) }{ N_0},  \frac{1}{N_1} \right\} } \right) \\
    &\overset{(a)}\leq \tilde{O}\left( C_{\mathtt{Alg}} \sqrt{ \frac{2}{ N_0/\mathtt{C}(\pi^*|\rho) + N_1 } } \right) 
\end{align*}
  where $(a)$ is the harmonic mean.  
\end{proof}

\begin{theorem}[Restatement of \Cref{thm:main regret}]
    Let $\mathtt{Alg}$ be a confidence-based algorithm satisfying and \Cref{eluder condition}. Then, the regret of  \Cref{alg:hybrid} is 
    \[
        \text{Regret}(N_1) = \tilde{O}\left( C_{\mathtt{Alg}} \sqrt{N_1}\sqrt{  \frac{ N_1 }{ N_0/\mathtt{C}(\pi^{-\varepsilon}|\rho)  + N_1}  } \right),
    \]
    where $\mathtt{C}(\pi^{-\varepsilon}|\rho)$ is the maximum concentrability coefficient of the sub-optimal policies whose sub-optimality gap is at least $\varepsilon$, and $\varepsilon=\tilde{O}(1/\sqrt{N_0 + N_1})$.
\end{theorem}

\begin{proof}

Define $\mathrm{Reg}_t = V_{\Mc^*}^{\pi^*} - V_{\Mc^*}^{\pi_t}$, where $\pi^*$ is an optimal policy. We mainly consider $t$ such that $\mathrm{Reg}_t > \varepsilon$, where $\varepsilon$ is decided later. By the definition of $\mathtt{Alg}$, we have
    \begin{align*}
        \mathrm{Reg}_t &= V_{\Mc^*}^{\pi^*} - V_{\Mc^*}^{\pi_t} \\
        & \leq   \hat{V}^{\pi^*} + \hat{\mathtt{U}}_{\mathtt{Alg}}( \pi^*|\Dc_0\cup\Dc_{t-1}) - \hat{V}^{\pi_t} + \hat{\mathtt{U}}_{\mathtt{Alg}} (\pi_t|\Dc_0\cup\Dc_{t-1}) \\
        &\leq 2  \hat{\mathtt{U}}_{\mathtt{Alg}} (\pi_t|\Dc_0\cup\Dc_{t-1})
    \end{align*}
    
By \Cref{def:Coverbility}, we have
\begin{align*}
     \hat{\mathtt{U}}_{\mathtt{Alg}} (\pi_t|\Dc_0\cup\Dc_{t-1}) & \leq   \hat{\mathtt{U}}_{\mathtt{Alg}} (\pi_t|\Dc_0) \\
     &\leq C_{\mathtt{Alg}}\mathtt{U}(\pi_t) \\
     &\leq C_{\mathtt{Alg}}\sqrt{ \mathtt{C}(\pi_t|\rho) }  \mathtt{U}( \rho |\Dc_0 ) \\
     & \leq \tilde{\Theta} \left(C_{\mathtt{Alg}} \sqrt{\frac{ \mathtt{C}(\pi^{-\varepsilon}|\rho) }{N_0}} \right),
\end{align*}
where $\mathtt{C}(\pi^{-\varepsilon}|\rho) = \max_{\pi: V_{P^*}^{\pi} < V_{P^*}^{\pi^*} -\varepsilon} \mathtt{C}(\pi|\rho)$. Therefore, by choosing $\varepsilon = O(1/\sqrt{N_0 + N_1})$, we have
    \begin{align*}
        \mathrm{Reg} &= \sum_{t=1}^{N_1}  \mathrm{Reg}_t \\
        & \leq \sum_{t=1}^{N_1}  \max\left\{ \tilde{O}\left( \min\left\{ C_{\mathtt{ALg}}  \sqrt{\frac{ \mathtt{C}(\pi^{-\varepsilon}|\rho) }{N_0}}, \hat{\mathtt{U}}_{\mathtt{Alg}}( \pi_t|\Dc_{t-1}) \right\} \right), \varepsilon \right\} \\
        &\leq \tilde{O} \left( \sum_{t=1}^{N_1} C_{\mathtt{ALg}} \sqrt{ \frac{2}{N_0/\mathtt{C}(\pi^{-\varepsilon}|\rho) + C_{\mathtt{Alg}}^2\hat{\mathtt{U}}_{\mathtt{Alg}}(\pi_t|\Dc_{t-1})^{-2} } } \right) + N_1\varepsilon \\
        & =  \tilde{O} \left( C_{\mathtt{ALg}} \sqrt{\frac{\mathtt{C}(\pi^{-\varepsilon}|\rho) }{N_0}} \sum_{t=1}^{N_1} \sqrt{ 1 - \frac{1}{C_{\mathtt{Alg}}^{-2}\hat{\mathtt{U}}_{\mathtt{Alg}}(\pi_t|\Dc_{t-1})^2 N_0/\mathtt{C}(\pi^{-\varepsilon}|\rho)  + 1 } } \right) + N_1\varepsilon \\
        & \overset{(a)}\leq \tilde{O} \left( C_{\mathtt{ALg}} \sqrt{\frac{\mathtt{C}(\pi^{-\varepsilon}|\rho)}{N_0}} \sqrt{N_1} \sqrt{ N_1 -  \sum_{t=1}^{N_1} \frac{1}{C_{\mathtt{Alg}}^{-2}\hat{\mathtt{U}}_{\mathtt{Alg}}(\pi_t|\Dc_{t-1})^2 N_0/\mathtt{C}(\pi^{-\varepsilon}|\rho)  + 1 } } \right) + N_1\varepsilon \\
        & \overset{(b)}\leq \tilde{O}\left( C_{\mathtt{ALg}} \sqrt{\frac{\mathtt{C}(\pi^{-\varepsilon}|\rho)}{N_0}} \sqrt{N_1} \sqrt{ N_1 - \frac{N_1^2 }{C_{\mathtt{Alg}}^{-2}\sum_{t=1}^{N_1 }\hat{\mathtt{U}}_{\mathtt{Alg}}(\pi_t|\Dc_{t-1})^2 N_0/\mathtt{C}(\pi^{-\varepsilon}|\rho)  + N_1 } } \right) + N_1\varepsilon \\
        & \overset{(c)}\leq \tilde{O} \left( C_{\mathtt{ALg}} \sqrt{\frac{\mathtt{C}(\pi^{-\varepsilon}|\rho)}{N_0}} \sqrt{N_1} \sqrt{ N_1 -  \frac{N_1^2 }{ N_0/\mathtt{C}(\pi^{-\varepsilon}|\rho)  + N_1 } } \right) + N_1\varepsilon \\
        & = \tilde{O}\left( C_{\mathtt{ALg}} \sqrt{N_1} \sqrt{  \frac{N_1 }{ N_0/\mathtt{C}(\pi^{-\varepsilon}|\rho)  + N_1 } } \right),
    \end{align*}
where $(a)$ and $(b)$ follow from the Cauchy's inequality, and $(c)$ is due to the eluder-type condition (\Cref{eluder condition}). 
    
\end{proof}

\section{Missing Proofs of Examples}\label{sec:example proof}

\subsection{Tabular MDPs}

We first show that the algorithm proposed in \citet{azar2017minimax} satisfies the eluder-type condition.

\begin{lemma}\label{lemma: tabular mdp satisfies eluder-type}
    The algorithm proposed in \citet{azar2017minimax} satisfies the eluder-type condition.
\end{lemma}
\begin{proof}
    First, we introduce several additional notations. Given a dataset $\Dc_t$, Let $\hat{\Mc}_t$ be the estimated model from \Cref{eqn:UCB-VI} and the corresponding counters are $N_{t,h}(x_h,a_h)$. We define another uncertainty function $\tilde{\mathtt{U}}(\pi|\Dc)$ as 

\[ \tilde{\mathtt{U}}(\pi|\Dc) = \Eb\left[ \sum_{h=1}^H \frac{\beta}{\sqrt{N_h(x_h,a_h)}} \big| \Mc^*, \pi \right],  \]  
where $\beta=\tilde{O}(H)$. According to Section B in \citet{azar2017minimax}, the difference of value functions under the true model $\Mc^*$ and the estimated model $\hat{\Mc}$ is also upper bounded by $\tilde{\mathtt{U}}$. Therefore,  
\begin{align*}
    \sum_{t=1}^{N_1} \tilde{\mathtt{U}}_{\mathtt{Alg}}(\pi_t|\Dc_t)^2 &\leq  \beta^2 \sum_{t=1}^{N_1} \Eb\left[ H \sum_{h=1}^H \frac{1}{N_{t,h}(x_h,a_h)}\middle|\Mc^*, \pi_t\right] \\
    & = \beta^2   H \sum_{h=1}^H \sum_{x_h,a_h} \Eb\left[ \sum_{t=1}^{N_1} \frac{\mathbbm{1}\{x_{t,h}=x_h, a_{t,h}=a_h\} }{N_{t,h}(x_h,a_h)}  \right] \\
    & = \beta^2  H \sum_{h=1}^H \sum_{x_h,a_h} \Eb\left[ \sum_{t=1}^{N_1} \frac{N_{t+1,h}(x_h,a_h) - N_{t,h}(x_h,a_h)}{N_{t,h}(x_h,a_h)} \right] \\
    &\leq \beta^2 H^2 |\Xc||\mathcal{A}| \log N_1 \\
    & = \tilde{\Theta}(H^4|\Xc||\mathcal{A}|).
\end{align*}
\end{proof}

While \Cref{lemma: tabular mdp satisfies eluder-type} shows that the eluder-type condition is satisfied, the result upper bound is loose. To obtain a tighter bound, we directly apply the inequalities in \citet{azar2017minimax} and prove \Cref{coro:tbl mdp}.

\begin{corollary}[Restatement of \Cref{coro:tbl mdp}]
    For tabular MDPs, under the hybrid RL framework in \Cref{alg:hybrid}, using $\hat{U}_{\mathtt{Alg}}(\pi|\Dc)$ defined in the RHS of \Cref{eqn: tabular MDP bonus informal}, the regret scales in 
    \[ \tilde{O}\left(\sqrt{H^4|\Xc||\mathcal{A}|N_1} \sqrt{\frac{N_1}{N_0/\mathtt{C}(\pi^{-\varepsilon}|\rho) + N_1}} \right); \]
    and the sub-optimality gap is 
    \[ \tilde{O}\left( \sqrt{ \frac{H^4|\Xc||\mathcal{A}|}{ N_0/\mathtt{C}(\pi^*|\rho) + N_1 } } \right).  \]
\end{corollary}

\begin{proof}

Due to \Cref{lemma: tabular mdp satisfies eluder-type}, we have that 
\begin{align*}
    \sum_{t=1}^{N_1} \hat{\mathtt{U}}_{\mathtt{Alg}}(\pi_t|\Dc_t) \leq \tilde{O}(\sqrt{H^4|\Xc||\mathcal{A}|N_1}).
\end{align*}

In addition, we show that $\frac{\hat{\mathtt{U}}_{\mathtt{Alg}}(\pi^*)}{\mathtt{U}_{\mathtt{\Mc^*}}(\pi^*)} \leq \sqrt{|\Xc||\mathcal{A}|}$. 

By Theorem 3 in \citet{xiong2022nearly}, there exists an MDP $\Mc^*$ such that
\begin{align*}
    \min_{\mathtt{Alg}}\Eb_{\Dc_0}\left[| V_{\Mc}^{\pi^*} - V_{\mathtt{Alg}}^{\pi^*} \right]| \geq \Omega\left( \sqrt{|\Xc||\mathcal{A}|} \Eb_{\pi^*}\left[ \sum_{h=1}^H \frac{1}{\sqrt{N(s_h,a_h)}} \right] \right)
\end{align*}
Thus, we have
\begin{align*}
    \frac{\hat{\mathtt{U}}_{\mathtt{Alg}}(\pi^*)}{\mathtt{U}_{\mathtt{\Mc^*}}(\pi^*)} \leq \tilde{O} \left(\frac{ \Eb_{\pi^*}\left[ \sum_{h=1}^H \frac{\beta }{\sqrt{N(s_h,a_h)}} \right]  }{ \sqrt{|\Xc||\mathcal{A}|} \Eb_{\pi^*}\left[ \sum_{h=1}^H \frac{1}{\sqrt{N(s_h,a_h)}} \right]  } \right) = O(\frac{H}{\sqrt{|\Xc||\mathcal{A}|}}) = O(\sqrt{H^4|\Xc||\mathcal{A}|})
\end{align*}

Therefore, the regret of \Cref{alg:hybrid} is 
\begin{align*}
    \mathrm{Regret}(N_1) &= \tilde{O}\left( \min \left\{ N_1\sqrt{\frac{\mathtt{C}(\pi^{-\varepsilon}|\rho)}{N_0}} , \sqrt{H^4|\Xc||\mathcal{A}|N_1} \right\}  \right) \\
    &\leq \tilde{O}\left(\sqrt{H^4|\Xc||\mathcal{A}| N_1} \sqrt{ \frac{1}{N_0/\mathtt{C}(\pi^{-\varepsilon}|\rho) + N_1} } \right).
\end{align*}

For the sub-optimality gap, a straightforward modification of the proof in \citet{azar2017minimax} shows that $\frac{1}{N_1}\sum_{t=1}^{N_1} \hat{\mathtt{U}}_{\mathtt{Alg}}(\pi_t|\Dc_t) \leq \tilde{O}(\sqrt{H^4|\Xc||\mathcal{A}|/N_1})$. Following the same proof in \Cref{thm:gap appendix} except without using the Cauchy's inequality, we complete the proof.

\end{proof}

\subsection{Linear Contextual Bandits}
First, we show that Lin-UCB satisfies the eluder-type condition. 
\begin{lemma}\label{lemma:bandits eluder appendix}
    Lin-UCB satisfies the eluder-type condition. 
\end{lemma}

\begin{proof}
    Note that $\hat{\mathtt{U}}_{\mathtt{Alg}}(\pi|\Dc) \leq \beta \Eb_{x\sim q^*} \|\Eb_{a\sim \pi(x)}[\phi(x,a)]\|_{\hat{\Lambda}^{-1}}$ according to \Cref{eqn:lcb ucb}, where $\beta = \tilde{O}(\sqrt{d})$

    Therefore, by the elliptical potential lemma~\citep{carpentier2020elliptical}, we have 
    \begin{align*}
        \hat{\mathtt{U}}_{\mathtt{Alg}}(\pi_t|\Dc_t)^2 & \leq \beta^2 \Eb \left[ \sum_{t=1}^{N_1}  \|\phi(x_t,a_t)\|_{\hat{\Lambda}_{t-1}^{-1}}^2  \right] \\
        &\leq \beta^2 d \log N_1\\
        & = \tilde{\Theta}(d^2).
    \end{align*}
\end{proof}

Combining \Cref{lemma:bandits eluder appendix} and \Cref{thm:suboptimality,thm:main regret}, we obtain the following corollary.
\begin{corollary}[Restatement of \Cref{coro:LinConBandits}]
    For linear contextual bandits, under the hybrid RL framework in \Cref{alg:hybrid}, using $\hat{\mathtt{U}}_{\mathtt{Alg}}$ as defined in \Cref{eqn:lcb ucb}, the regret is
    \[ \tilde{O}\left(d\sqrt{N_1}\sqrt{\frac{N_1}{N_0/\mathtt{C}(\pi^{-\varepsilon}|\rho) + N_1}} \right);\]
    and the sub-optimality gap is
    \[ \tilde{O}\left( d\sqrt{\frac{1}{N_0/\mathtt{C}(\pi^{-\varepsilon}|\rho) + N_1} }\right) .\]
\end{corollary}

\section{Missing Proofs in Section 5}\label{sec:lower bound proof appendix}
In this section, we provide the full analysis for lower bounds.

\subsection{Hard Instance}\label{sec:hard instance}
First, we introduce the notation of truncated Gaussian distributions. If a Gaussian random variable $X\sim N(0,I_d)$ is truncated to $\{x:\|x\|_2\leq r\}$, then we denote the truncated Gaussian distribution of $X$ as $N(0,I_d|r)$. 

In the lower bound analysis, we follow the setting in \citet{he2022reduction}, as specified below. 

{\bf Arms and dimension:} There are 2 arms: $\{1,2\}$, and the feature dimension is 2. 

{\bf Feature vectors and the context distribution:} The feature vector of the second arm is always $0$. For the feature vector of the first arm, let the distribution of the context $x$ satisfy that each $\phi(x,1)\in\Rb^2$ is sampled from a truncated normal $N(0,I_2|1)$. 

{\bf Model parameter:} The model parameter $\theta^{*} \in\Rb^2$ is sampled uniformly from a sphere $\mathbb{S}_r=\{x\in\Rb^2:\|x\|=r\}$, where $r\in[0,1/\sqrt{d}]$. The constraint on $r$ is due to the boundedness assumption that $\|\theta^*\|\leq 1 $. 

{\bf Additional notations:} Recall that $\Dc_{t} = \{x_{\tau}, a_{\tau}, r_{\tau}\}_{\tau<t}$ is the online dataset. Here $r_{t}$ is sampled from a sub-Gaussian distribution with mean $\phi(x_t,a_{t})^{\TT}\theta^*$ and variance $1$. We further denote that $\phi(x_{t},a) = x_{t,a}$. We re-parameterize $\theta^*$ by its angle, i.e. $\theta^* = r(\cos\gamma^*,\sin\gamma^*)^{\TT}$, and $\gamma^*$ is sampled uniformly from the interval $[0,2\pi)$. We further denote $e_1 = (1,0)$ and $e_2 = (0,1)$, which form the canonical basis of $\Rb^2$.

With the aforementioned setting, we present the generic regret lower bound modified from Proposition 3.5 in \citet{he2022reduction}. 

\begin{theorem}[ ]\label{thm: Regret to estimation error-appendix}
    For the hard instance described in \Cref{sec:hard instance}, if available dataset is $\Dc$, the sub-optimality gap is lower bounded by 
    \begin{align*}
    \Omega \left(\mathop{\inf}_{\theta\in\mathcal{F}(\Dc) } \frac{1}{r}\Eb_v\left[ \left\|\theta^*-  \theta\right\|^2 \right] \right);
    \end{align*}
    if $\Dc_0\cup\Dc_t$ is the available dataset at episode $t$,
    regret can be lower bounded by
    \begin{align*}
    \Omega \left(\sum_{t\in[T]}\mathop{\inf}_{\theta_{t}\in\mathcal{F}(\Dc_0\cup\Dc_t) } \frac{1}{r}\Eb_v\left[ \left\|\theta^*-  \theta_{t}\right\|^2 \right] \right).
    \end{align*}
\end{theorem}

\subsection{Proof of Lower Bound}\label{sec:CDP w. memory}
Equipped with \Cref{thm: Regret to estimation error-appendix}, we are able to lower bound the regret by the estimation error.

{\it Proof Outline:} {\bf Step 1} is to decompose the estimation error to the expectations a random variable $Z$ which capture the covariance of the estimator. {\bf Step 2} upper bounds $\Eb[Z]$. {\bf Step 3} combines the previous steps to prove the results.

{\bf Step 1: Decompose the Estimation Error.}

\begin{lemma}[Fingerprinting Lemma]
\label{lemma: Fingerprinting}
    Define random variables $Z$ as follows.
    \begin{align*}
    &Z = (\hat{\theta} - \theta^*)^{\TT}(-e_1\sin\gamma^* + e_2\cos\gamma^*)(-e_1\sin\gamma^* + e_2\cos\gamma^*)^{\TT}\bar{V}(\bar{\theta} - \theta^*) ,
\end{align*}
where $\bar{V} = \Lambda_0 + \sum_{\tau<t}\phi(x_\tau,1)\phi(x_\tau,1)^{\TT}$, and $\bar{\theta} = \bar{V}^{\dagger}\left( \sum_{(x,a,r)\in\Dc_0}\phi(x,a)r + \sum_{\tau<t}\phi(x_\tau,1)r_{\tau}\mathbbm{1}\{a_\tau=1\}\right)$, and recall that $r_{\tau}$ is sampled from $\mathcal{N}(\phi(x_\tau,a_\tau)^{\TT}\theta^*,1)$.

Then, we have
\begin{align*}
    \Eb\left[ \left\|\theta^*-  \hat{\theta}\right\|^2 \right] = 2r^2 - 2r^2 \Eb[Z] .
\end{align*}
\end{lemma}

\begin{proof}
Due to $\|\theta^*\| = \|\hat{\theta}\| = r$, it suffices to analyze the term $\Eb\left[\hat{\theta}^{\TT}\theta^*\right].$ Note that $\theta^* = r(\cos\gamma^*,\sin\gamma^*)^{\TT}$. Then, we have
\begin{align*}
\Eb\left[\hat{\theta}^{\TT}\theta^*\right]& = \frac{r}{2\pi} \int_0^{2\pi} e_1^{\TT}\Eb[\hat{\theta}|\gamma^*]\cos\gamma^* + e_2^{\TT}\Eb[\hat{\theta}|\gamma^*]\sin\gamma^* d\gamma^*\\
    & = \frac{r}{2\pi} \left(e_1^{\TT}\Eb[\hat{\theta}|\gamma^*]\sin\gamma^* - e_2^{\TT}\Eb[\hat{\theta}|\gamma^*]\cos\gamma^* \right)\bigg|_{\gamma^*=0}^{\gamma^*=2\pi}\\
    &\quad - \frac{r}{2\pi}\int_0^{2\pi} e_1^{\TT}\frac{\partial}{\partial\gamma^*}\Eb[\hat{\theta}|\gamma^*]\sin\gamma^* + e_2^{\TT}\frac{\partial}{\partial\gamma^*}\Eb[\hat{\theta}|\gamma^*]\cos\gamma^* d\gamma^*\\
    & = r\Eb_{\gamma^*}\left[ \left(-e_1\sin\gamma^* + e_2\cos\gamma^*\right) ^{\TT}\frac{\partial}{\partial\gamma^*}\Eb[\hat{\theta}|\gamma^*] \right].
\end{align*}

For the derivative, it is worth noting that $\Eb[\hat{\theta}|\gamma^*] = \Eb \left[\Eb\left[\hat{\theta}\big|\Dc_0\cup \Dc_t\right] \big|\gamma^*\right]$. We have
\begin{align*}
    \frac{\partial}{\partial\gamma^*}\Eb[\hat{\theta}|\gamma^*] &= \int_{\Dc_0\cup\Dc_t}\Eb\left[\hat{\theta}\big| \Dc_0\cup\Dc_t \right]\frac{1}{(2\pi)^{(t-1)/2}} \frac{\partial}{\partial\gamma^*} \exp\left(-\frac{1}{2} \sum_{(x,a,r)\in\Dc_0\cup\Dc_t}\left(r-\phi(x,a)^\TT\theta^*\right)^2  \right)\\
    & = r\Eb \left[ \Eb\left[\hat{\theta}|\Dc_0\cup \Dc_t\right] (-e_1\sin\gamma^* + e_2\cos\gamma^*)^{\TT}   \sum_{(x,a,r)\in\Dc_0\cup\Dc_t}\phi(x,a)(r-\phi(x,a)^\TT\theta^*)     \bigg| \theta^*\right]\\
    & = r\Eb \left[ \hat{\theta} (-e_1\sin\gamma^* + e_2\cos\gamma^*)^{\TT} \bar{V}(\bar{\theta}-\theta^*) \big| \theta^*\right].
\end{align*}

Combining with the fact that $\Eb[\bar{V}(\bar{\theta} - \theta^*)|\theta^*, \bar{V}] = 0$, we have
\begin{align*}
    \Eb\left[\hat{\theta}^{\TT}\theta^*\right] &= r^2\Eb\left[(-e_1\sin\gamma^* + e_2\cos\gamma^*)^{\TT}\left(\hat{\theta} - \theta^*\right)(-e_1\sin\gamma^* + e_2\cos\gamma^*)^{\TT} \bar{V}(\bar{\theta} - \theta^*)\right]\\
    & = r^2 \Eb[Z].
\end{align*}

Therefore,
\begin{align*}
    \Eb\left[\left\|\theta^* - \hat{\theta}\right\|^2\right] & = 2r^2 - 2\Eb\left[\hat{\theta}^{\TT}\theta^*\right] = 2r^2 - 2r^2 \Eb[Z],
\end{align*}
which completes the proof.

\end{proof}

{\bf Step 2: Upper Bound Each $\Eb[Z]$.}

\begin{lemma}\label{lemma: upper bound Z}
    Under the same setting as in \Cref{lemma: Fingerprinting}, we have
    \begin{align}
        &\Eb[Z] \leq \sqrt{(t-1 + N_0 \Eb_{x\sim q^*, a\sim\rho(\cdot|x)}[(\phi(x,a)^\TT\theta_{\bot}^*)^2])\Eb\left[\|\hat{\theta} - \theta^*\|^2\right]}. \label{eqn: upper bound Zi non-private}
    \end{align}
\end{lemma}

\begin{proof}

Recall that 
\begin{align*}
Z_i =  (\hat{\theta} - \theta^*)^{\TT}(-e_1\sin\gamma^* + e_2\cos\gamma^*)(-e_1\sin\gamma^* + e_2\cos\gamma^*)^{\TT} \sum_{(x,a,r)\in\Dc_0\Dc_t} \phi(x ,a )(r  - \phi(x , a )^{\TT} \theta^* ).
\end{align*}

By the Cauchy's inequality, we have
\begin{align*}
    \Eb[Z]^2 &\leq \Eb\left[\|\hat{\theta} - \theta^*\|^2\right] \Eb\left[\left(\big(-e_1\sin\gamma^* + e_2\cos\gamma^*\big)^{\TT} \sum_{(x,a,r)\in\Dc_0\cup\Dc_t} \phi(x ,a )(r - \phi(x ,a )^{\TT}\theta^*) \right)^2\right] \\
    & = \Eb\left[\| \hat{\theta} - \theta^*\|^2\right] \Eb\left[ \sum_{ (x,a,r)\in\Dc_0\cup\Dc_t} \left(\big(-e_1\sin\gamma^* + e_2\cos\gamma^*\big)^{\TT} \phi(x,a) \right)^2\right]\\
    & \leq (t-1 + N_0 \Eb_{x\sim q^*, a\sim\rho(\cdot|x)}[(\phi(x,a)^\TT\theta_{\bot}^*)^2])\Eb\left[\|\hat{\theta} - \theta^*\|^2\right].
\end{align*}

\end{proof}

{\bf Step 3: Lower Bound the Total Regret.}

\begin{theorem}[Restatement of \Cref{thm:lowerbound}]\label{thm:lowerbound-appdix}
    Under the instance  described in \Cref{sec:hard instance}, any hybrid RL algorithm must incur a sub-optimality gap in 
    \(\Omega \left( \frac{1}{\sqrt{N_0/\mathtt{C}(\pi^*|\rho) + N_1}} \right),\)
    and regret in 
    \(
    \Omega\left( \frac{N_1}{\sqrt{N_0/\mathtt{C}(\pi^{-\varepsilon}|\rho) + N_1}}\right).\)
\end{theorem}

\begin{proof}
Combine Step 1 (\Cref{lemma: Fingerprinting}) and Step 2 (\Cref{lemma: upper bound Z}), we have
\begin{align*}
    2r^2 &= \Eb\left[\|\hat{\theta} -\theta^*\|^2\right] + 2r^2 \Eb\left[Z \right]\\
    &\leq  \Eb\left[\|\hat{\theta} -\theta^*\|^2\right] + 2r^2\sqrt{(N_0\alpha + t-1) \Eb\left[\|\hat{\theta}-\theta^*\|^2\right]},
\end{align*}
where $\alpha = \Eb_{x\sim q^*, a\sim\rho(\cdot|x)}[(\phi(x,a)^\TT\theta_{\bot}^*)^2]$.

Therefore,
\begin{align*}
    r^2\geq \Eb\left[\|\hat{\theta} -\theta^*\|^2\right]\geq \frac{r^2}{4r^2 (N_0\alpha + t-1) + 4}.
\end{align*}

Substituting the above result into the generic lower bound in \Cref{thm: Regret to estimation error-appendix}, and selecting $r=1/\sqrt{N_0\alpha + N_1}$, we conclude that
\begin{align*}
    \text{Regret}(T) &\geq \Theta\left(\sum_{t\in[T]} \frac{r}{ 4 r^2 (N_0\alpha + t-1) + 4} \right)\\
    & = \Theta\left(  \frac{1}{4r} \log  \left( 1 + \frac{ r^2 N_1 }{1 + r^2 N_0 \alpha} \right) \right)\\
    & = \Theta \left( \frac{r N_1 }{1 + r^2 N_0\alpha} \right) \\
    & = \Theta \left( \frac{N_1 }{\sqrt{N_0\alpha + N_1}} \right) .
\end{align*}

To establish the relationship between $\alpha$ and the concentrability coefficient, we repeatedly apply \Cref{lemma: Fingerprinting} and \Cref{lemma: upper bound Z} on dataset $\Dc_0$. Then, we obtain
    \begin{align*}
        r^2 \geq \Eb[ \|\hat{\theta} - \theta^*\|^2 ] \geq \frac{r^2}{ 4r^2 N_0 \alpha  + 4}.
    \end{align*}

By choosing $r = 1/\sqrt{N_0\alpha}$, we have $ \frac{1}{r}\Eb[\|\hat{\theta} - \theta^*\|^2 ]  = \Theta(1/\sqrt{N_0\alpha})$. Thus, 
    \begin{align*}
    \alpha^{-1} &= \Theta\left( \frac{ \frac{1}{r} \Eb[\|\hat{\theta} - \theta^*\|^2 ]  }{ 1/\sqrt{N_0} } \right)^2 \\
    & = \Theta \left( \frac{ \mathtt{U}(\pi^*|\Dc_0) }{ \mathtt{U}(\rho|\Dc_0) } \right)^2\\
    & = \mathtt{C}(\pi^*|\rho).
    \end{align*}

 The proof is completed by noting that $\mathtt{U}(\pi^*|\Dc_0) = \Theta(\mathtt{U}(\pi^{-\varepsilon}|\Dc_0))$ for any $\varepsilon = O(1/\sqrt{N_0+N_1})$ in the setting described in \Cref{sec:hard instance}.   
\end{proof}

\section{Additional Experiments}\label{append:experiment}

In this section, we provide additional experimental results evaluating the performance of algorithms instantiated within our proposed framework in more realistic environments. Specifically, we consider a contextual linear bandit constructed from the MovieLens dataset~\citep{harper2015movielens} and a tabular MDP discretized from the Mountain Car environment ~\citep{Moore90efficientmemory-based} implemented in Gymnasium~\citep{towers2024gymnasium}.

\subsection{Experimental Results with MovieLens Dataset}
\textbf{Environment.} We construct our linear contextual bandit environment using the MovieLens-100K dataset~\citep{harper2015movielens}, which provides sparse ratings from 943 users on 1682 movies. Following \citet{bogunovic2021stochastic}, we first apply collaborative filtering~\citep{Keval2019Collaborative} to complete the partially observed rating matrix. We then factorize the resulting rating matrix $R = [r_{i,a}]\in \Rb^{943\times1682}$ using non-negative matrix factorization with $3$ latent factors, yielding $R = XH$. Here, $X \in \Rb^{943\times3}$ represents user feature vectors, and $H \in \Rb^{3\times1682}$ represents movie feature vectors. In the linear contextual bandit framework, we treat each row in $X$ (i.e., each user’s feature vector) as the context, denoted by $x_i\in \Rb^{3}$ for the $i$-th user. The contexts are known in the contextual bandits. Meanwhile, we randomly select 20 columns of $H$ to serve as the arms (i.e., 20 movies for chosen), where each arm’s unknown parameter vector corresponds to a movie’s feature vector that must be estimated from data, denoted as $\theta_a\in\Rb^{3}$. At each decision point, the linear contextual bandit model randomly provides a user context and the agent predicts the expected reward for choosing arm $a$ (i.e., recommending a movie) based on observed context $x$.

\textbf{Offline Dataset Collection.} Similar to the offline data collection method in the main text, we adopt the Boltzmann policy~\citep{szepesvari2022algorithms} as our behavior policy. Specifically, the policy chooses an action $a$ according to
\[
  \rho(a | x) = \frac{\exp\{k r(x,a)\}}{\sum_{b \in \mathcal{A}} \exp\{k r(x,b)\}},
\]
where $k\in\mathbb{R}$, $\mathcal{A}$ is the set of all arms (movies), and $r(x,a)$ denotes the true reward function. Under the Boltzman policy, a larger $k$ makes $\rho$ closer to the greedy policy (i.e., always selecting the highest reward arm), whereas a smaller $k$ makes the policy more exploratory.

The three different behavior policies, $\rho_1$, $\rho_2$, and $\rho_3$, are constructed in the same way introduced in the main text, each defined by a distinct value of $k$ in the Boltzmann distribution. As $k$ decreases, the policy’s action deviates from the optimal choice, increasing the concentrability coefficient $\mathtt{C}(\pi^* | \rho)$ and thus degrading coverage of the optimal policy. \Cref{fig:append:MovieLens}(a) provide the exact $k$ values and the approximated $\mathtt{C}(\pi^{*} | \rho)$ in the environment. 

In \Cref{fig:movie_1} and \Cref{fig:movie_3}, we fix the offline dataset size $N_0$ to be $4000$ and vary $k$ to illustrate the impact of different levels of coverage. In constrast, in \Cref{fig:movie_2} and \Cref{fig:movie_4}, we fix the behavior policy to be $\rho_2$ while varying the offline dataset size. This setup allows us to systematically examine how both the behavior policy’s level of optimality and the sample size affect the performance of various learning algorithms.

\begin{figure*}[h!]

\begin{minipage}[b]{.19\linewidth}
    \centering
    \begin{tabular}{c|c|c}
        $\rho$ & $\mathtt{C}(\pi^*|\rho)$ & $k$ \\
        \hline
        $\rho_1$  & $1.0$ & $\infty$ \\
        $\rho_2$  & $2.43$ & $5$      \\
        $\rho_3$  & $4.05$ & $0$
        \end{tabular}
    \subcaption{CE in bandits}
  \end{minipage}
  \hfill
  \begin{minipage}[b]{.19\linewidth}
    \centering
    \includegraphics[width=0.95\linewidth]{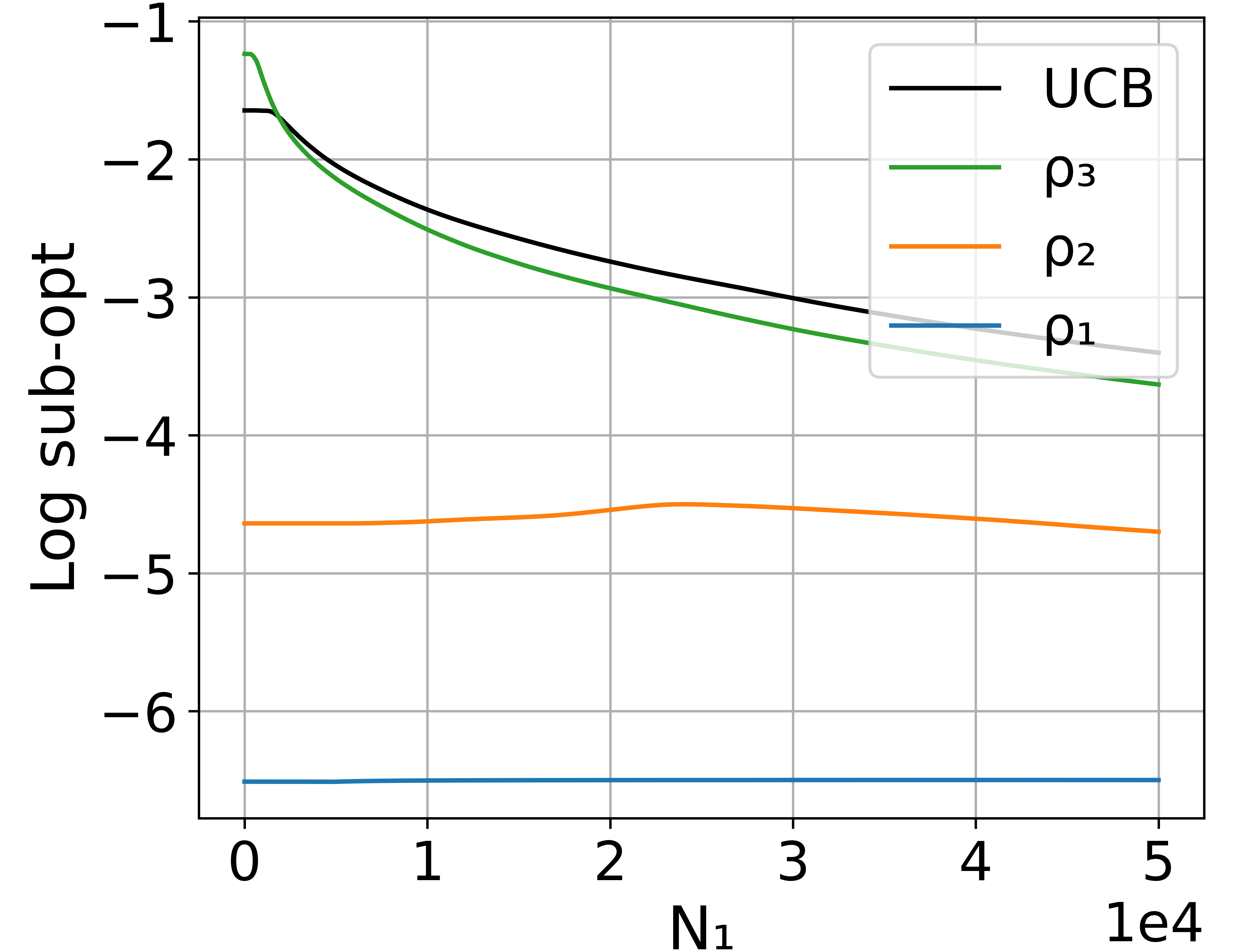}
    \subcaption{SOG v.s. $\rho$} 
    \label{fig:movie_1}
  \end{minipage}
  \hfill
  \begin{minipage}[b]{.19\linewidth}
    \centering
    \includegraphics[width=0.95\linewidth]{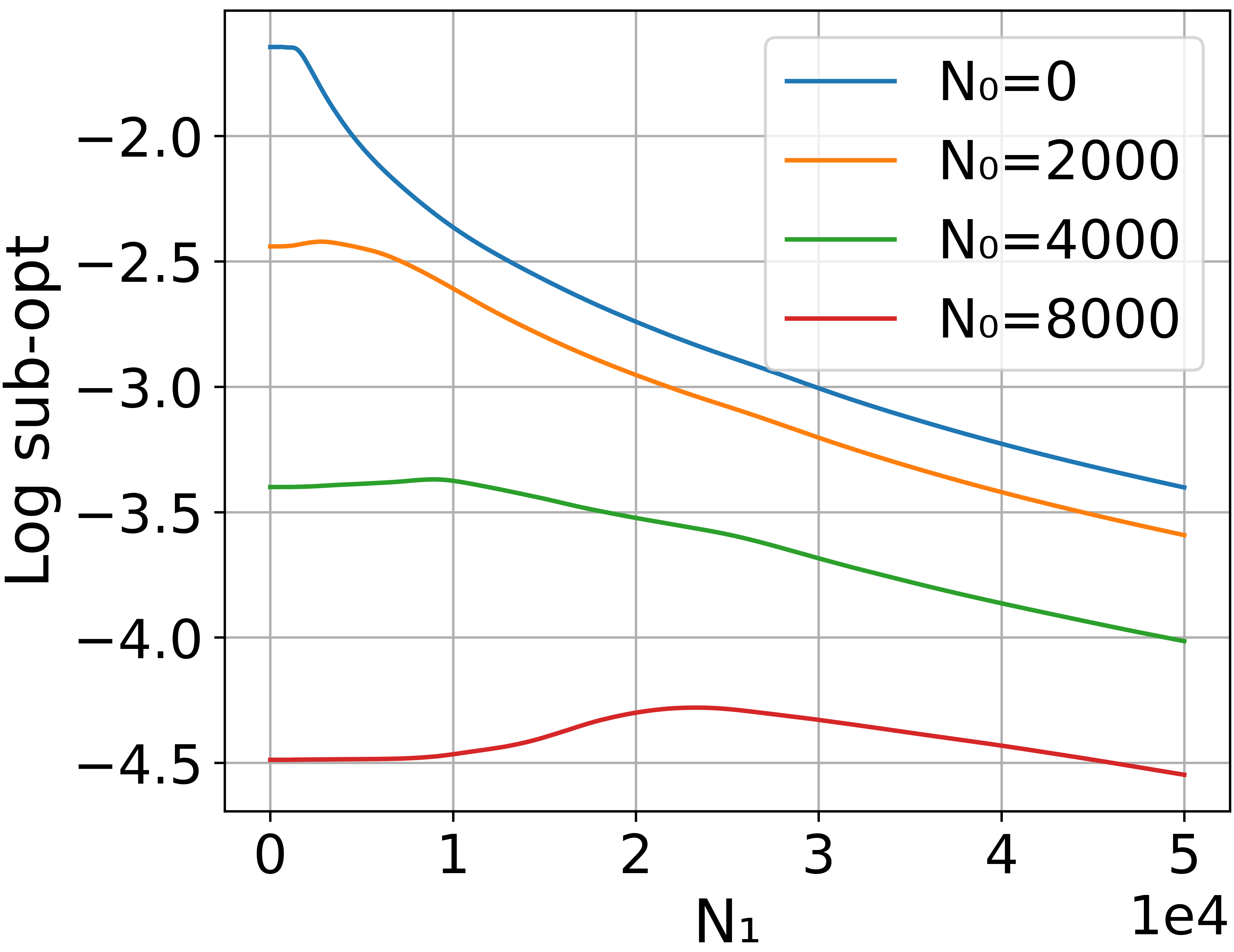} 
    \subcaption{SOG v.s. $N_0$}
    \label{fig:movie_2}
  \end{minipage}
  \hfill
  \begin{minipage}[b]{.19\linewidth}
    \centering
    \includegraphics[width=0.95\linewidth]{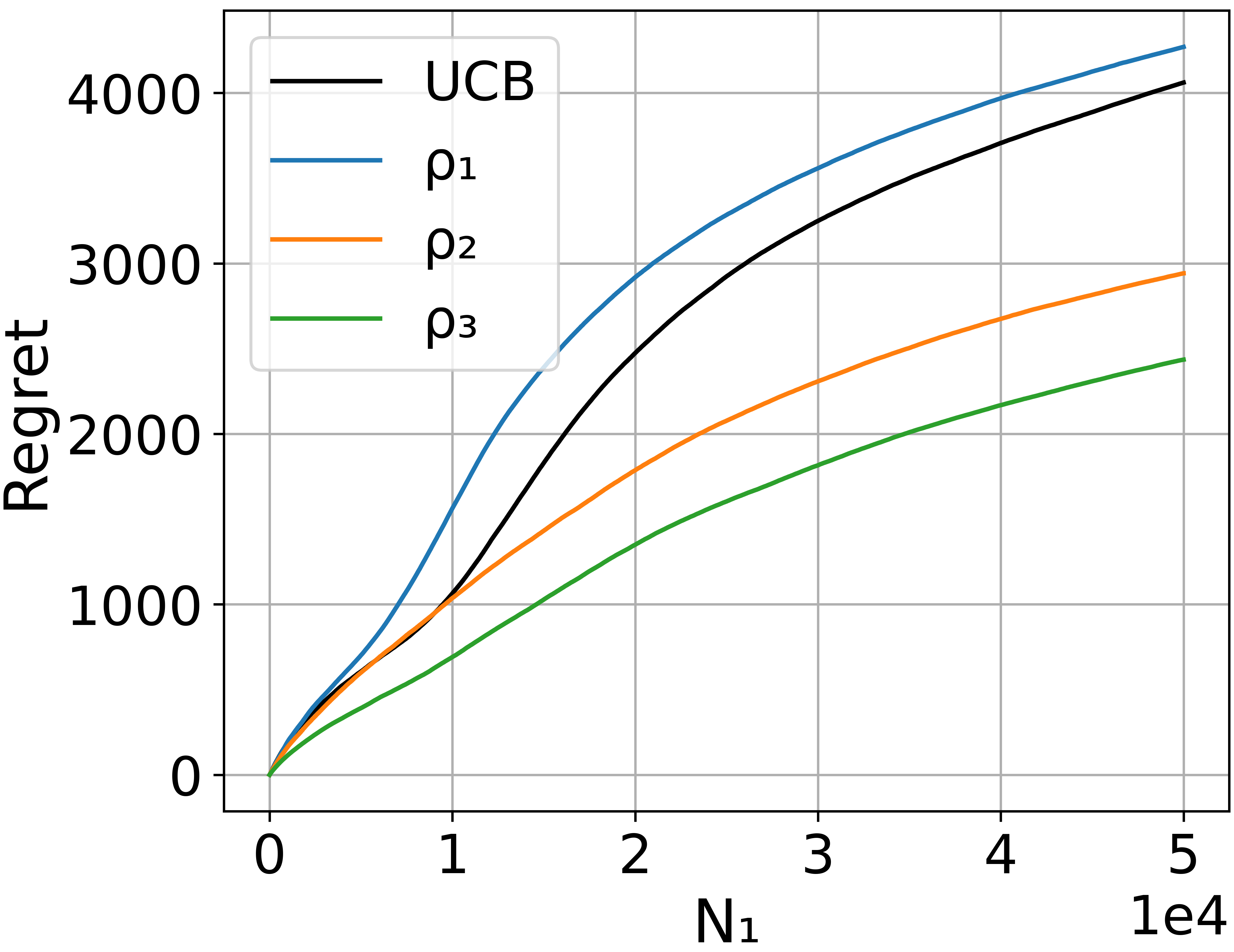} 
    \subcaption{Regret v.s. $\rho$}
    \label{fig:movie_3}
  \end{minipage}
   \hfill
  \begin{minipage}[b]{.19\linewidth}
    \centering
    \includegraphics[width=0.95\linewidth]{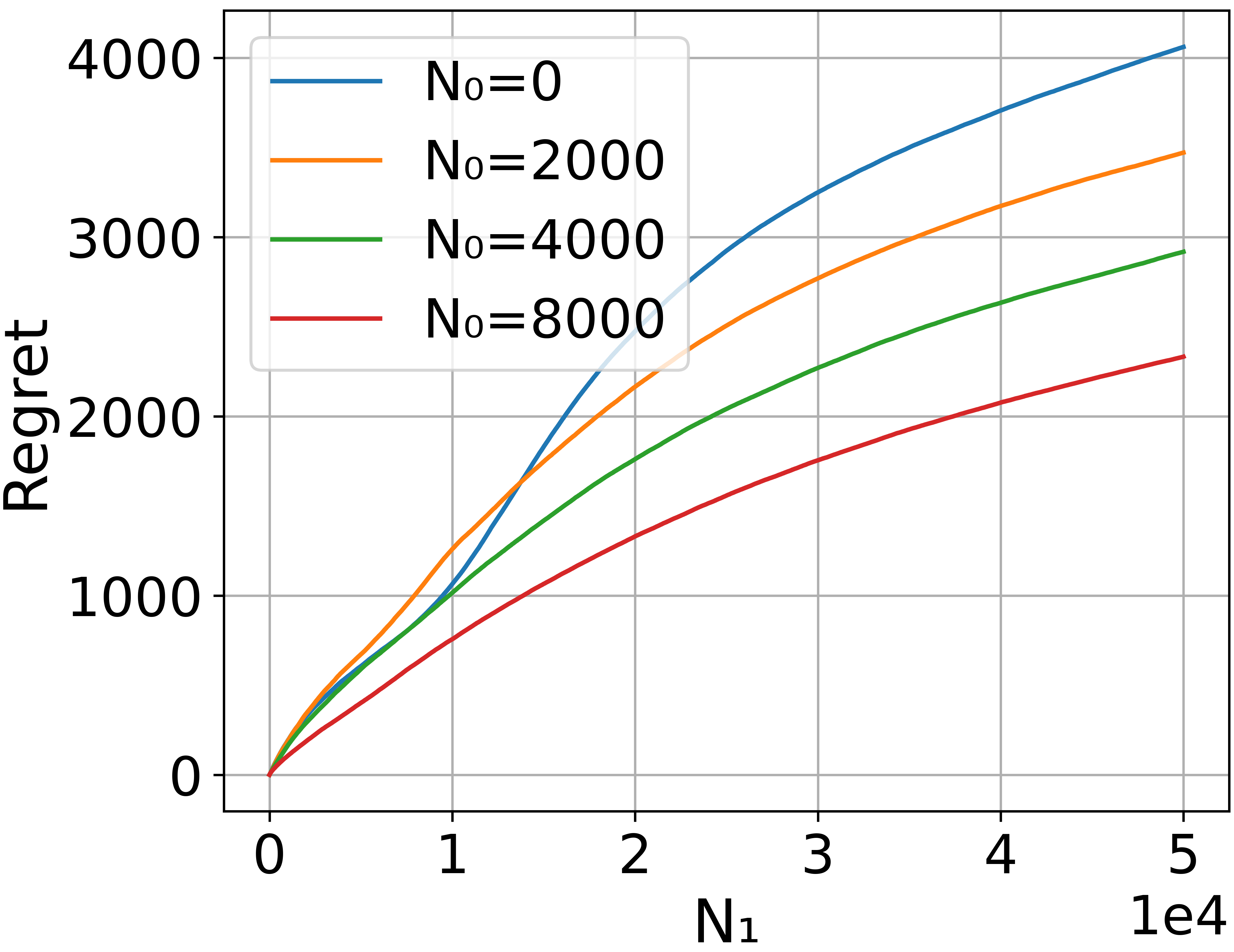} 
    \subcaption{Regret v.s. $N_0$}
    \label{fig:movie_4}
  \end{minipage}
  \caption{Experimental results on sub-optimality gap (SOG) and regret for different behavior policies and $N_0$. Figure (a) shows the concentrability coefficients (CE) of three different behavior policies in MovieLens linear contextual bandits.
  }
  \label{fig:append:MovieLens}
\end{figure*}

\textbf{Results.}
Figure~\ref{fig:append:MovieLens} compares sub-optimality gaps and regrets under various offline behavior policies and dataset sizes, using the pure UCB algorithm without offline data as the baseline. Overall, the results mirror the main text’s theoretical and empirical findings. 

For gap minimization, policies with smaller concentrability coefficients $\mathtt{C}(\pi^*|\rho)$ or larger offline samples lead to tighter gaps, reaffirming that offline data focused on covering optimal actions can greatly enhance efficiency. 

In regret minimization, the algorithm benefits from the offline dataset. It benefit from the offline data more, if the offline data has diverse offline coverage generated by larger $\mathtt{C}(\pi^*|\rho)$ and larger offline data size, as shown in \Cref{fig:movie_3} and \Cref{fig:movie_4}. 
A slight exception occurs when the offline dataset is collected by optimal actions. In this case, the offline data has very big $\mathtt{C}(\pi^{-\epsilon}|\rho)$ and may encourage the algorithm to explore the sub-optimal arms first. Nonetheless, the outcomes are still consistent with theoretical predictions, highlighting the distinct offline data requirements for gap minimization versus regret minimization.

\subsection{Experimental Results with Mountain Car Environment}

\textbf{Environment.} The Mountain Car environment\citep{Moore90efficientmemory-based} is a classic benchmark task in which an underpowered car must drive up a steep slope, featuring a continuous state space over position $[-1.2,0.6]$ and velocity $[-0.07,0.07]$ and a discrete action set for accelerating forward, backward and no acceleration. To model the Mountain Car environment and implement a UCB-type algorithm, we first discretize the state space and apply the UCB algorithm on the tabular MDP. Specifically, we designate any state with position exceeding 0.5 as the goal state, while all other states are formed by uniformly discretizing the position range $[-1.2,0.5]$ and velocity range $[-0.07,0.07]$ into 30 equal intervals each, yielding 901 discrete states in total. The agent receives a reward of 1 only upon taking an action from the goal state; otherwise, the reward is 0. After an action is taken in the goal state, the environment is reset to its start configuration, then follows the original Mountain Car transition dynamics.

\textbf{Offline Data Collection.}
Different from Boltzmann policy-based offline data collection, we use \Cref{alg:append:MountCarOffline}, which iteratively interleaves exploration and exploitation to generate an offline dataset for the Mountain Car environment. Specifically, at each iteration, a model $\hat{P}$ is used to estimate two Q functions, $\hat{Q}_b(s,a)$ and $\hat{Q}_r(s,a)$. Here, $b(s,a)$ is an exploration bonus function akin to a UCB term~\citep{Auer:2002}, encouraging broader exploration, whereas $r(s,a)$ is the known reward function driving exploitation. These estimates yield two policies, exploration-focused $\hat{\pi}_b$ and exploitation-focused $\hat{\pi}_r$. Trajectories collected under these policies populate two datasets, $D$ and $D'$, respectively. After each round of data collection, both the model $\hat{P}$ and the bonus function $b(s,a)$ are updated, reflecting the optimism-in-the-face-of-uncertainty principle characteristic of UCB-based methods. After 10{,}000 iterations, trajectories from $D$ and $D'$ are combined using the offline coefficient $\alpha$, thus balancing exploration and exploitation in the final offline dataset. 

\begin{algorithm}[t]
    \caption{Mountain Car Offline Data Collection}
    \label{alg:append:MountCarOffline}
    \begin{algorithmic}[1]
        \STATE \textbf{Input:} Coefficient $\alpha\in[0,1]$, number of fffline trajectories $N_0 \le 10000$, discount factor $\gamma=0.99$.
        \STATE \textbf{Initialization:} $\Dc \gets \emptyset$, $\Dc' \gets \emptyset$, $\hat{P}$ as a uniform transition model, $b(s,a) \gets 1$
        \FOR{$i = 1$ \textbf{to} $10000$}
            \STATE Estimate 
            \[
              \hat{Q}_b(s_0,a_0) = \hat{\mathbb{E}}_{\hat{P}}\bigl[\sum_t \gamma^t b(s_t,a_t)\mid s_0,a_0\bigr], 
            \]
            \[
              \hat{Q}_r(s_0,a_0) = \hat{\mathbb{E}}_{\hat{P}}\bigl[\sum_t \gamma^t r(s_t,a_t)\mid s_0,a_0\bigr]
            \]
            \STATE Derive policies $\hat{\pi}_b^i$ from $\hat{Q}_b$ and $\hat{\pi}_r^i$ from $\hat{Q}_r$
            \STATE Collect trajectories $\tau_i$ under $\hat{\pi}_b^i$ and $\tau'_i$ under $\hat{\pi}_r^i$
            \STATE Update $\Dc \gets D \cup \{\tau_i\}$ and $\Dc' \gets D' \cup \{\tau'_i\}$
            \STATE Update $\hat{P}$ and $b(s,a)$ using $\Dc$
        \ENDFOR
        \STATE Sample $\alpha \, N_{0}$ trajectories from $\Dc'$ and $(1-\alpha)\,N_0$ trajectories from $\Dc$ to form $\Dc_{0}$
        \STATE \textbf{Output:} Offline dataset $\Dc_0$
    \end{algorithmic}
\end{algorithm}

The motivation for this approach is that, in the Mountain Car environment, a purely uniform policy tends to remain confined to the valley, failing to explore higher positions effectively. This leads to inadequate coverage of the state-action space. By contrast, the pure online exploration policy $\hat{\pi}_b$ naturally seeks out all state-action pairs and thereby achieves broader coverage, populating $\Dc$ with a wide range of trajectories. Meanwhile, the pure exploitation policy $\hat{\pi}_r$ focuses on maximizing rewards and populates $\Dc'$ with near-optimal behavior. Finally, the offline coefficient $\alpha$ determines how these two datasets are combined into the final offline dataset $\Dc_0$, so \Cref{alg:append:MountCarOffline} can simulate the offline dataset with different coverage for our hybrid learning. Higher $\alpha$ generates better coverage on the optimal policy and lower $\alpha$ generates better coverage on all policies. Remarkably, this offline dataset collection method does not contradict our setting that the offline data should be collected by one fixed policy. The offline dataset can be viewed as being collected by a mixture policy that randomly samples $\{\hat{\pi}^i_b\}_i$ and $\{\hat{\pi}^i_r\}_i$ for $N_0$ times. However, it is difficult to provide the estimated $\mathtt{C}(\pi^*|\rho)$ for each policy or data distribution. Intuitively, a greater offline coefficient $\alpha$ results in smaller $\mathtt{C}(\pi^*|\rho)$ and greater $\mathtt{C}(\pi^{-\epsilon}|\rho)$; a smaller offline coefficient $\alpha$ results in greater $\mathtt{C}(\pi^*|\rho)$ and smaller $\mathtt{C}(\pi^{-\epsilon}|\rho)$.

\begin{figure*}[h!]

  \begin{minipage}[b]{.24\linewidth}
    \centering
    \includegraphics[width=0.95\linewidth]{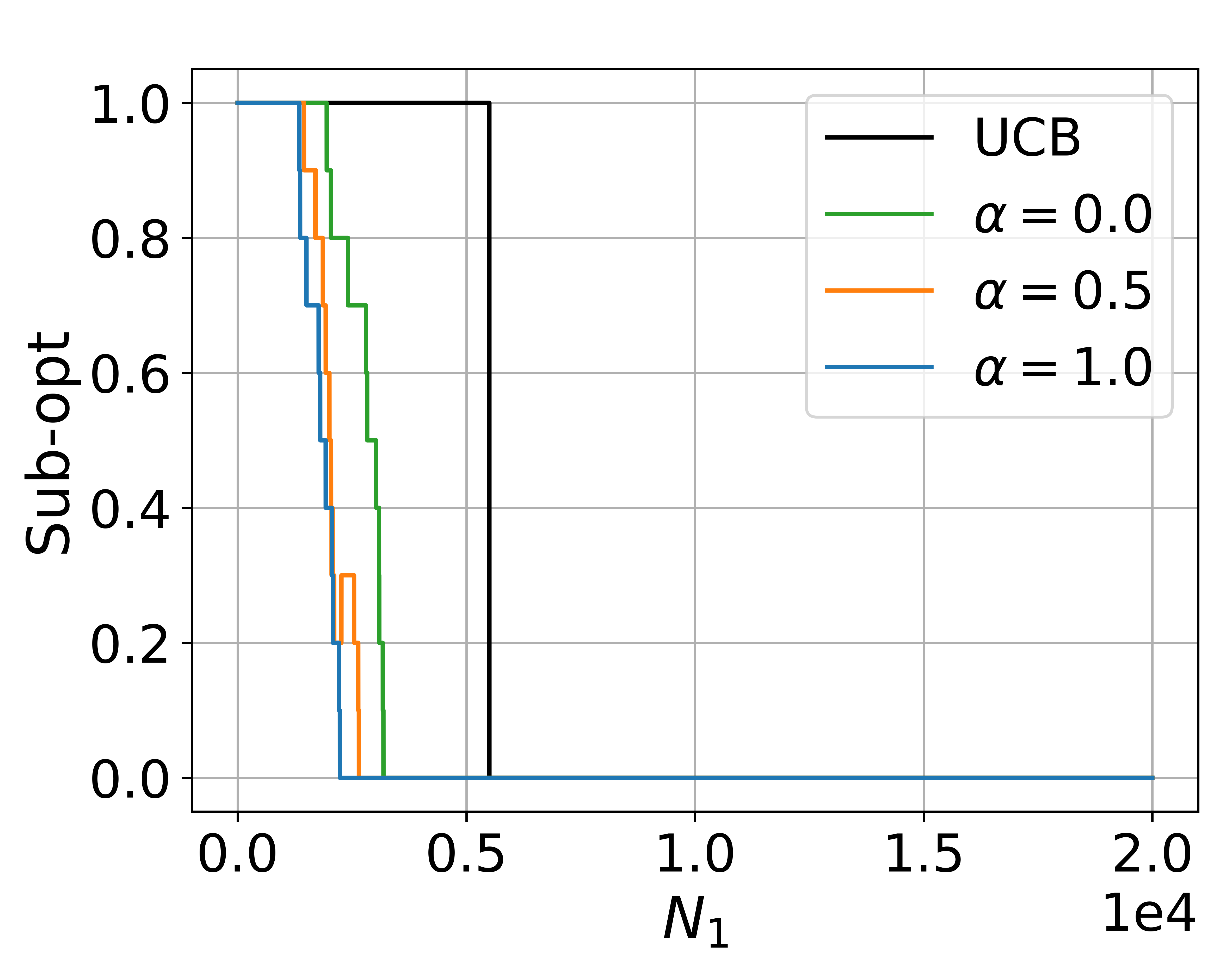}
    \subcaption{SOG v.s. $\alpha$} 
    \label{fig:MountCar_1}
  \end{minipage}
  \hfill
  \begin{minipage}[b]{.24\linewidth}
    \centering
    \includegraphics[width=0.95\linewidth]{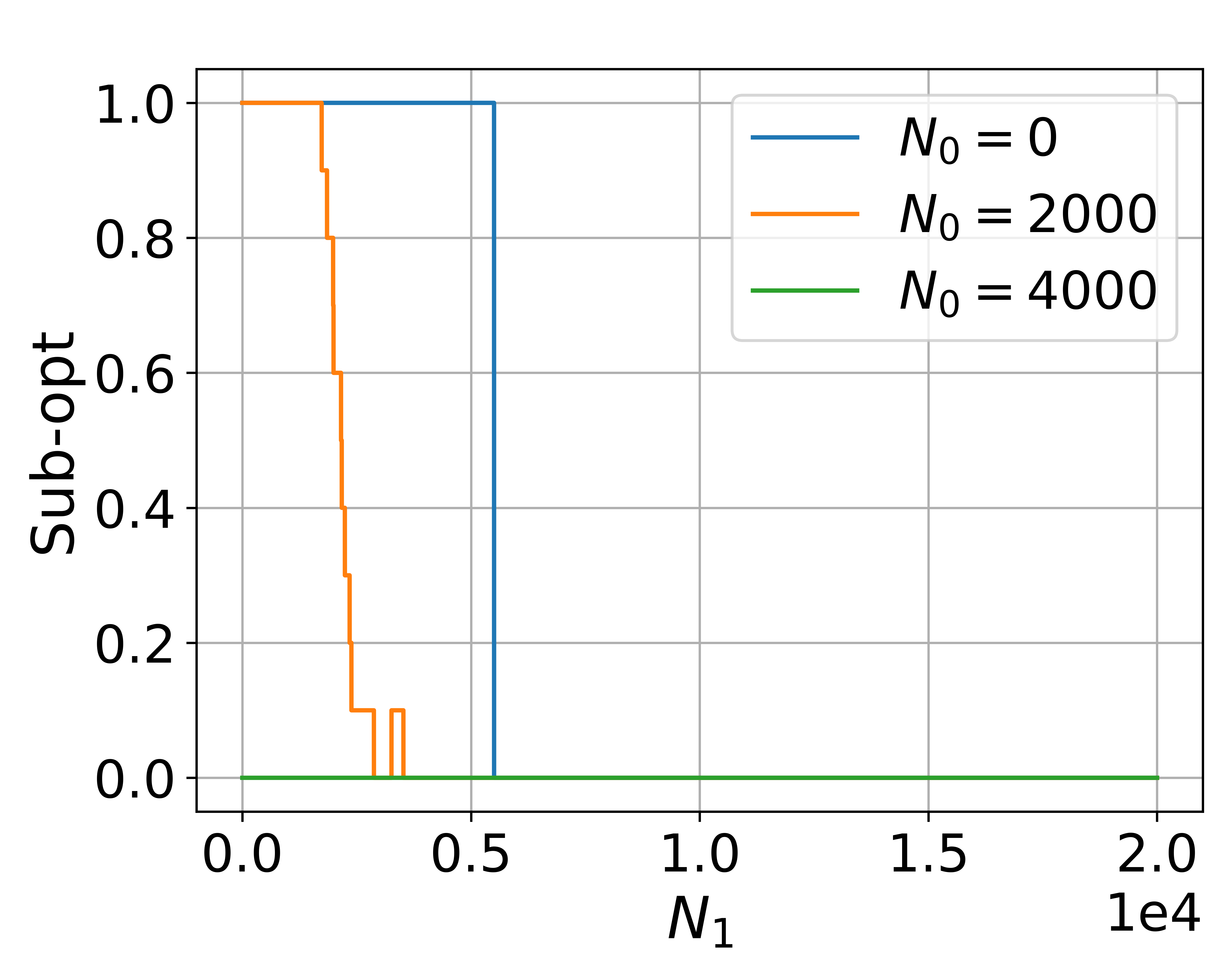} 
    \subcaption{SOG v.s. $N_0$}
    \label{fig:MountCar_2}
  \end{minipage}
  \hfill
  \begin{minipage}[b]{.24\linewidth}
    \centering
    \includegraphics[width=0.95\linewidth]{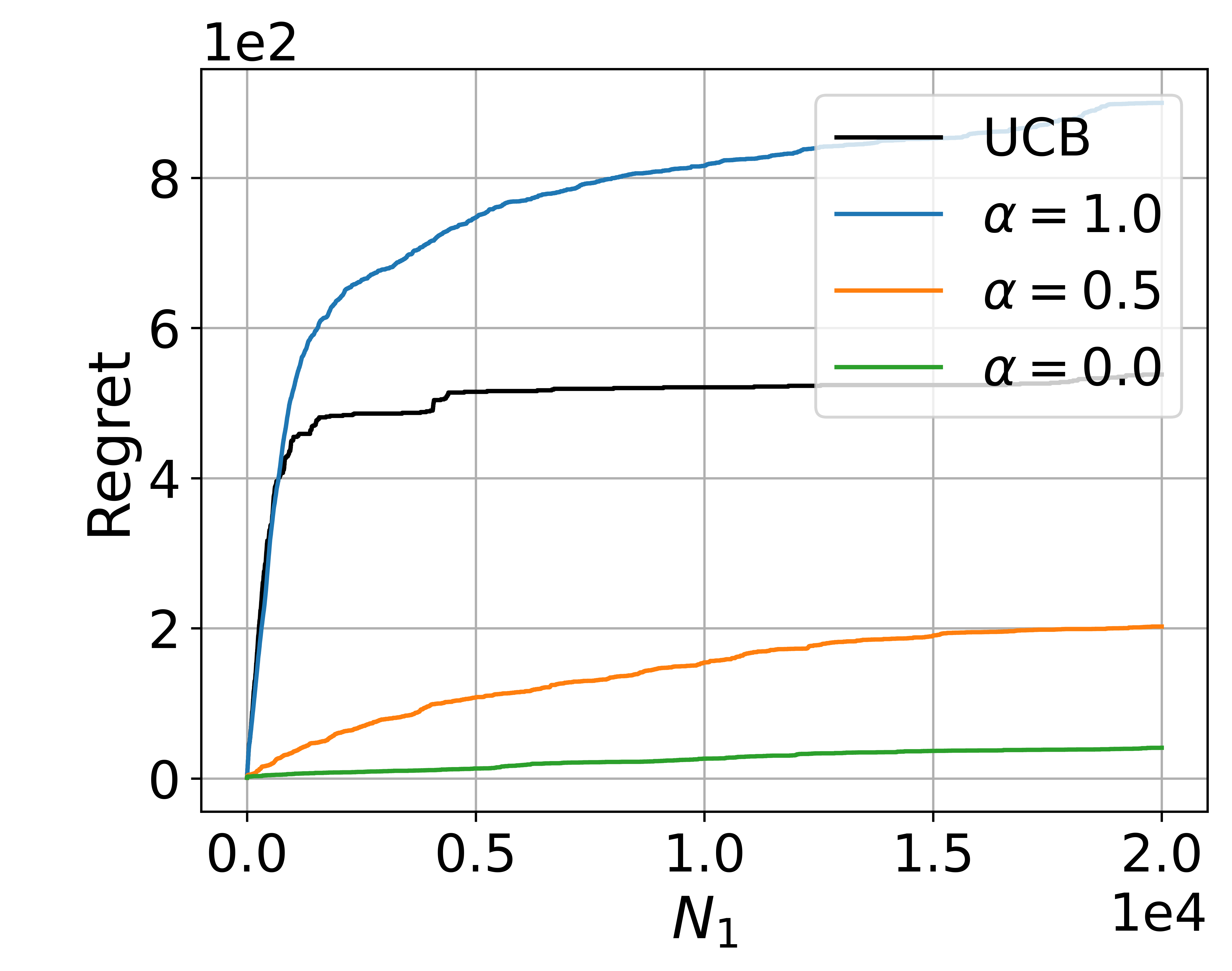} 
    \subcaption{Regret v.s. $\alpha$}
    \label{fig:MountCar_3}
  \end{minipage}
   \hfill
  \begin{minipage}[b]{.24\linewidth}
    \centering
    \includegraphics[width=0.95\linewidth]{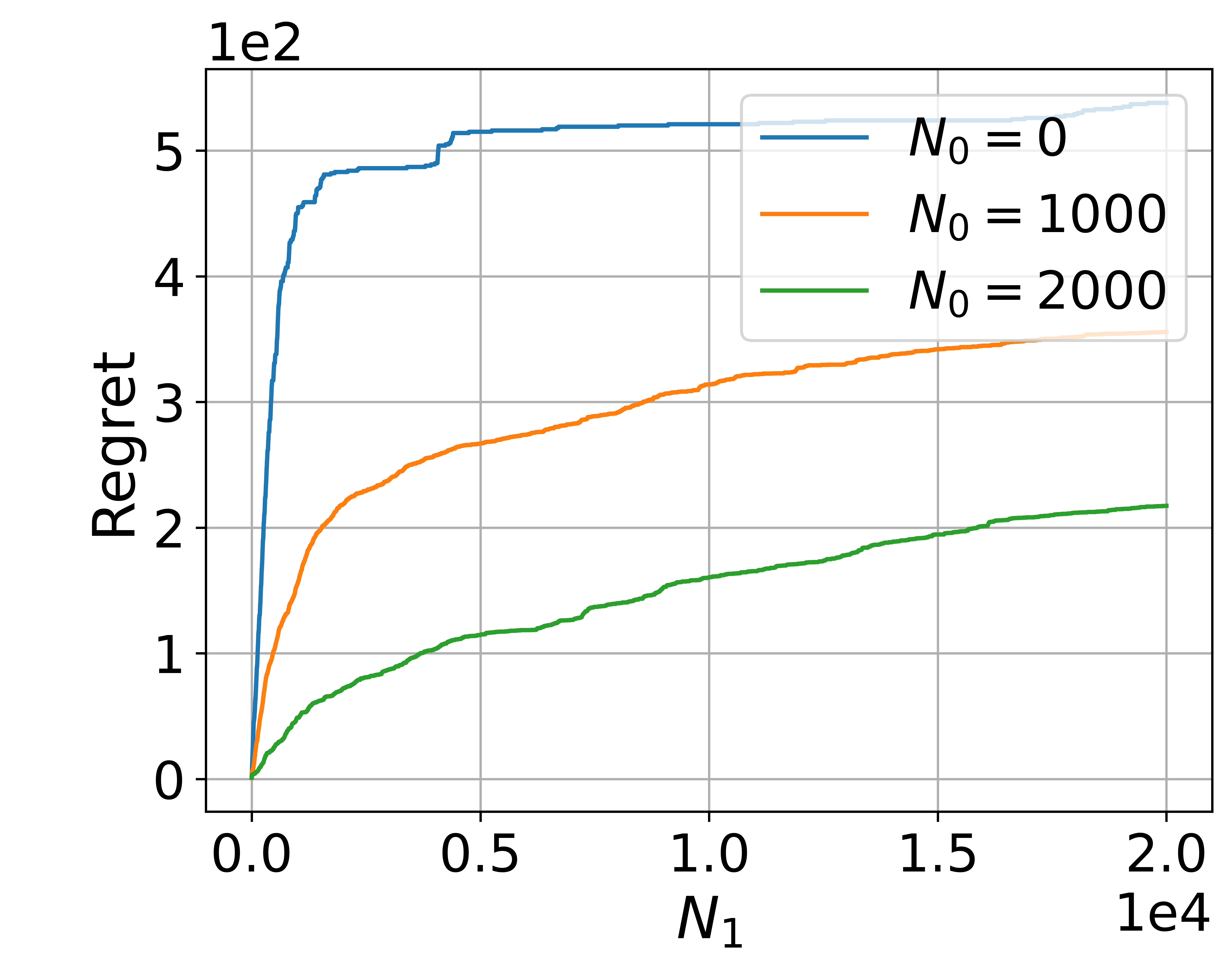} 
    \subcaption{Regret v.s. $N_0$}
    \label{fig:MountCar_4}
  \end{minipage}
  \caption{Experimental results on sub-optimality gap (SOG) and regret for different offline coefficient $\alpha$ and number of offline trajectories $N_0$. }
  \label{fig:append:MountCar}
\end{figure*}


\textbf{Results.} We summarize our findings in Figure~\ref{fig:append:MountCar}, where we experiment with various values of the offline coefficient $\alpha$ and different sizes of the offline dataset in the Mountain Car environment. For each configuration, we conduct 10 independent runs and track the mean sub-optimality gap or regret as a function of the online time horizon $N_1$. The pure UCB method without any offline data serves as our baseline. In \Cref{fig:MountCar_1} and \Cref{fig:MountCar_3}, the number of offline trajectories are 2000. In \Cref{fig:MountCar_2} and \Cref{fig:MountCar_4}, $\alpha$ is $0.5$.

For sub-optimality gap minimization, our results in \Cref{fig:MountCar_1} and \Cref{fig:MountCar_2} indicate that incorporating offline data significantly improves performance compared to the pure online approach, validating the theoretical insights. In \Cref{fig:MountCar_1}, when $\alpha$ is larger, the offline dataset emphasizes more trajectories generated by the near-optimal policy, reducing the sub-optimality gap more quickly. Also, in \Cref{fig:MountCar_2}, increasing the size of the offline dataset further accelerates this reduction in sub-optimality.

For regret minimization, the results also show the benefit of incorporating the offline dataset. As in other settings, an offline dataset that prioritizes a near-optimal policy (i.e., higher $\alpha$) can sometimes lead to slightly higher regret in the experiments, since the hybrid algorithm devotes exploration efforts to less-visited actions, which may be sub-optimal. Conversely, when $\alpha$ is smaller, the offline dataset covers a broader range of behaviors, leading to more informed exploration for regret minimization as \Cref{fig:MountCar_3}. Also, in \Cref{fig:MountCar_4}, enlarging the offline dataset size lowers the regret curve in each case, aligning with our theoretical predictions.

\end{document}